\documentclass[onefignum,onetabnum]{siamart171218} 

\usepackage{amsmath} 
\usepackage{amssymb}  
\usepackage{amsfonts}
\usepackage{graphicx}
\usepackage{amsfonts}
\usepackage{siunitx}

\usepackage{multirow}
\usepackage{algpseudocode}
\usepackage{booktabs}

\def\hyph{-\penalty0\hskip0pt\relax}

\usepackage[framemethod=tikz]{mdframed}
\usepackage{enumitem}
\newlist{Henum}{enumerate}{1}
\setlist[Henum]{label=\textbf{H\arabic*}, leftmargin=2.2em, itemsep=0.25\baselineskip}

\newmdenv[
  skipabove=\baselineskip,
  skipbelow=\baselineskip,
  leftline=true, rightline=false, topline=false, bottomline=false,
  linecolor=black!30, linewidth=2pt,
  innerleftmargin=10pt, innerrightmargin=6pt, innertopmargin=6pt, innerbottommargin=6pt,
  roundcorner=2pt,
]{hypobox}

\newlist{SPenum}{enumerate}{1}
\setlist[SPenum]{label=\textbf{SP\arabic*}, leftmargin=2.2em, itemsep=0.25\baselineskip}

\newcommand{\norm}[1]{\left\lVert#1\right\rVert}

\ifpdf
\hypersetup{
  pdftitle={Physics-informed source identification},
  pdfauthor={B. Shaffer, B. Kinch, M. A. Hsieh and N. Trask}
}
\fi

\title{\bf PHYSICS-INFORMED SENSOR COVERAGE THROUGH STRUCTURE PRESERVING DATA-DRIVEN MODELS}

\author{Benjamin D. Shaffer, Brooks Kinch, Joseph Klobusicky, M. Ani Hsieh, and Nathaniel Trask}

\begin{document}

\maketitle
\thispagestyle{plain}
\pagestyle{plain}

\begin{abstract}
We present a machine learning framework for adaptive source localization in which agents use a structure-preserving digital twin of a coupled hydrodynamic-transport system for real-time trajectory planning and data assimilation. The twin is constructed with conditional neural Whitney forms, coupling the numerical guarantees of finite element exterior calculus with transformer based operator learning. The resulting model preserves discrete conservation, and adapts in real time to streaming sensor data. It employs a conditional attention mechanism to identify: a reduced Whitney-form basis; reduced integral balance equations; and a source field, each compatible with given sensor measurements. The induced reduced-order environmental model retains the stability and consistency of standard finite-element simulation, yielding a physically realizable, regular mapping from sensor data to the source field. We propose a staggered scheme that alternates between evaluating the digital twin and applying Lloyd’s algorithm to guide sensor placement, with analysis providing conditions for monotone improvement of a coverage functional. Using the predicted source field as an importance function within an optimal-recovery scheme, we demonstrate recovery of point sources under continuity assumptions, highlighting the role of regularity as a sufficient condition for localization. Experimental comparisons with physics-agnostic transformer architectures show improved accuracy in complex geometries when physical constraints are enforced, indicating that structure preservation provides an effective inductive bias for source identification.
\end{abstract}




\section{Abstract Problem Definition}

We consider a physical system governed by a steady-state coupled hydrodynamic–passive scalar transport process:
\begin{align}
    \label{eq:ps1}
    \mathcal{M}(\mathbf{v}) &= 0, \\
    \nabla \cdot \mathcal{F}(u, \mathbf{v}) &= f \quad \text{in } \Omega,
\end{align}
where $\mathcal{M}$ denotes a hydrodynamic model that generates the velocity field $\mathbf{v}$ (e.g., satisfying Navier–Stokes or Darcy), and $\mathcal{F}$ is the transport flux of a conserved density $u$, whose divergence balances the source $f$. The system is subject to appropriate boundary conditions on $\partial\Omega$ for both hydrodynamics and transport. 
 
We consider a set of agents $X = \{x_i(t)\}_{i=1}^N$ that provide pointwise sensor measurements $z_i = u(x_i(t))$. We formulate a distributed control problem in which sensor dynamics are prescribed to optimally estimate an unknown source field $f$. Central to this control problem is the properties of the inverse map $\mathbf{z} \mapsto (u,f)$. Our hypothesis is that structure-preserving, data-driven models encoding reduced-order predictions of Eq.~\eqref{eq:ps1} yield improved performance compared to physics-agnostic surrogates which treat the mapping $\mathbf{z} \mapsto (u,f)$ exclusively as a supervised regression problem.

We achieve this objective by constructing an adaptable, deployable \emph{digital twin}, a data-driven reduced-order model supporting real-time inference and data assimilation. Specifically, we introduce by ansatz a conservation law of the form
\begin{equation}
    \mathcal{H}(u; \theta, z) := \nabla \cdot \mathcal{F}_\theta(u, z) - f_\theta(z) = 0,
\end{equation}
where $\mathcal{F}_\theta(u, z)$ is a learnable flux and $f_\theta(z)$ is a learnable source term, both conditioned on sensor input $z$; we generically denote learnable operators throughout using $\theta$. To construct a reduced-order finite element exterior calculus (FEEC) basis compatible with this ansatz, we employ our recently developed Conditional Neural Whitney Forms (CNWF)~\cite{kinch2025structure}. The reduced basis, flux, and source are trained offline in concert to obtain a parametric model that (1) can be solved in real time and (2) can specialize these three learnable components to different sensor readings.

Specifically, we assume the flux operator $\mathcal{F}$ and the source field $f$ are unknown, and that we have access only to sparse and noisy measurements of the state $u$, obtained through an observation function $\mathcal{O}$, yielding sensor data $z = \mathcal{O}(u)$. We distinguish offline training and online deployment. Offline, we assume access to data pairs of the state and observations, $\{(z_i,u_i,f_i)\}_{i=1}^D$. During deployment, each observation $z$ conditions the learned physics, recalibrating the PDE operator in real time without retraining and thereby fulfilling the closed-loop requirements of a digital twin~\cite{willcox2023foundational}.

To identify the parameters $\theta$, we solve a PDE-constrained optimization problem:
\begin{equation}
    \min_\theta \; \mathcal{L}(u,\theta, z) \quad \text{subject to} \quad \mathcal{H}(u; \theta, z) = 0,
\end{equation}
where $\mathcal{L}$ is a loss functional that penalizes prediction error and enforces consistency with the sensor data. Data can be continuously generated by sampling random sensor configurations and providing compatible $(u,f)$ fields.

Our digital twin approach supports \emph{adaptive sensor placement} by leveraging the predicted source field $f_\theta(z)$ to inform the repositioning of mobile sensors; we prescribe a policy for agents of the form
\begin{equation}
    \dot{x}_i = G(f),
\end{equation}
where $F$ is a prescribed \textit{optimal coverage functional} designed to move sensors toward the source while maintaining quasi-uniform coverage over the support of $f$. This establishes a closed-loop interaction between observation, inference, and data acquisition. The model can then be refined through active feedback. 

We consider two hypotheses:
\begin{hypobox}
\begin{Henum}
    \item \label{hyp:H1} Physics-based constraints in the digital twin impose inductive biases that induce a more regular and accurate inverse mapping from sensor data to the source field.
    \item \label{hyp:H2} In the optimal coverage problem, regularity in the sensor-to-source map improves source recovery.
\end{Henum}
\end{hypobox}
\section{Motivation}
\label{sec:intro}

Many sensing tasks of scientific and engineering interest, such as locating pollution sources, tracking environmental hazards, or monitoring industrial processes, require measurements in dynamic and uncertain environments. In these settings, static sensor networks are often inefficient or infeasible, and the ability to reposition sensors in response to evolving knowledge of the environment is critical. Mobile sensor platforms, when paired with \emph{deployable digital twins}, that is, computational models that can be updated in real time using incoming data, establish an iterative loop of prediction, measurement, and control \cite{willcox2023foundational}. This capability improves sensing efficiency and guides sensor placement toward informative regions. This work develops the mathematical and algorithmic foundations for such an integrated system, focusing on source identification problems in which the digital twin directly informs sensor placement. This is of particular interest given rapid advances in neural digital twins where model inference is feasible on lightweight hardware \cite{kinch2025structure}.

We consider the \emph{optimal coverage} problem for \emph{source identification}, where mobile sensors must arrange themselves to identify the location of an unknown source term. Such problems involve sparse, indirect, and noisy observations of the transported field. Accurate reconstruction requires inverting the transport process, which is generally ill-posed. Traditional methods rely on prior knowledge, typically the exact form of the transport dynamics, the geometry, and possible source configurations \cite{bal2007inverse}. Because transport dynamics smooth and displace the influence of the source, the inverse problem is often highly ill-posed and sensitive to observation quality \cite{isakov1990inverse}. In many real-world settings, dense sensing is infeasible, and the ability to adaptively allocate limited mobile sensors is essential for accurate estimation \cite{khodayi2019model, luo2019distributed}. Mobile sensors provide a cost-efficient and adaptable means of data collection in dynamic or unknown environments; for example, in the Earth’s oceans, where it is infeasible to establish a dense network of seafaring drones.

Classical inverse methods, such as variational formulations and PDE-constrained optimization, provide principled tools for source reconstruction, but they typically assume full-field measurements, detailed knowledge of the underlying system, and require handcrafted regularization to mitigate ill-posedness \cite{engl2015regularization}. More recently, data-driven and machine learning approaches have emerged as flexible alternatives, using neural networks to approximate forward or inverse maps directly from partial observations \cite{arridge2019solving}. However, such models often lack interpretable parameterization (i.e., are treated as black boxes) \cite{rowan2025definition, linardatos2020explainable}, or impose governing equations only as soft constraints, leading to violations of physical consistency \cite{trask2022enforcing}. In inverse settings, this can yield implausible reconstructions, poor generalization under sparse data, and limited interpretability, posing challenges for high-consequence applications. Moreover, most machine-learned surrogates are not easily integrated with adaptive sampling or real-time feedback, limiting their use in adaptive sensing tasks.

We propose a structure-preserving framework for inverse source identification that directly embeds the conservation laws of the governing equations into a learned surrogate PDE model, without assuming a fixed functional form. Rather than learning the inverse operator, we learn both the source term and a nonlinear flux describing the transport process on a reduced finite element basis, and reconstruct the scalar field by solving the coupled conservation equation. We employ the recently proposed Conditional Neural Whitney Forms method (CNWF)~\cite{kinch2025structure}. All fields are discretized using finite element exterior calculus (FEEC) \cite{arnold2006finite, arnold2018finite}, ensuring compatibility with geometry and exact discrete conservation \cite{actor2024data, trask2022enforcing}. The source is represented as a normalized spatial density over a reduced partition-of-unity basis, yielding interpretable estimates compliant with conservation structure. This formulation explicitly couples the observed scalar field and the unknown source through a shared PDE constraint that regularizes the inverse problem and preserves physical structure throughout inference.

\begin{figure}[htbp]
    \centering
    \includegraphics[width=0.95\textwidth]{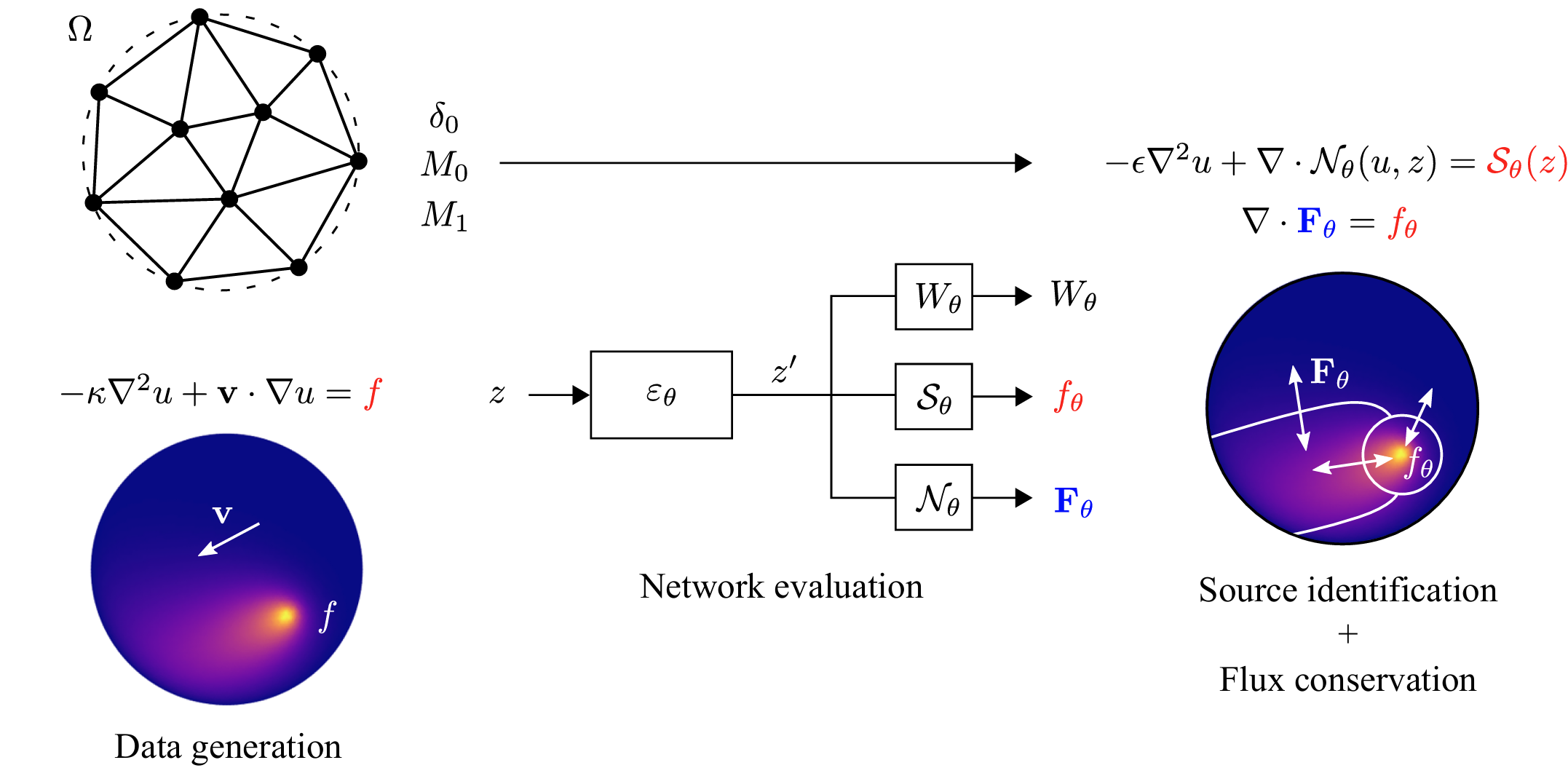}
    \caption{Overview of the source estimation procedure with the Conditional Neural Whitney Forms method. We learn a surrogate PDE over a reduced basis encoding control volumes, with a learnable nonlinear flux and source term. Sensing data are encoded via a self-attention transformer encoder to condition the basis, flux, and source. The resulting system is solved using a real-time nonlinear finite element solver and inherits the structure-preserving properties of the FEEC discretization. This provides a physically grounded interpretation of the learned source and directly couples it to the observed field data via the PDE, improving source prediction while remaining consistent with the unknown governing dynamics. This digital twin is then used to perform a modified version of Lloyd's algorithm for optimal coverage, directing sensors to locations informed by physics.}
    \label{fig:method_overview}
\end{figure}

\begin{sloppypar}
Recent transformer-based operator learning approaches lack explicit physical structure, relying instead on their powerful approximation properties to achieve accurate predictions. Our method integrates a transformer encoder into a finite element framework, combining the flexibility of attention-based models with the guarantees of FEEC. Conditioning enables the pretrained PDE operator to adapt in real time to sensor data, similar to how generative computer vision models sample images from a distribution conditioned on a text prompt. We supervise on both the scalar field and source term over training datasets to ensure consistency with the governing PDE, while downstream tasks use only the learned source term. An overview is given in Figure \ref{fig:method_overview} and explained in detail in Section \ref{sec:methodology_pde_surrogate}. A deep explanation of the FEEC framework is beyond the scope of this paper, and we direct interested readers to \cite{kinch2025structure}.
\end{sloppypar}

To enable a closed-loop connection between inference and sensing, we interpret the learned source field as an importance distribution that guides sensor placement via a geodesic implementation of Lloyd’s algorithm \cite{lloyd1982least} for nonconvex domains. This yields a two-scale optimization scheme: at each outer iteration, sensors are repositioned to minimize expected distance under the predicted importance field, and the model is updated with new observations from the new sensor locations. We provide a theoretical basis establishing improvements in prediction accuracy, model-relative coverage, and true coverage objectives under suitable regularity and accuracy conditions. This defines a closed feedback loop in which improved sensing leads to better predictions, which in turn guide sensors toward more informative configurations.

We validate our method on prototypical passive scalar transport systems. In all cases, our approach outperforms standard baselines, including an identical transformer model lacking structure preservation. We show that the proposed model improves both the inverse problem and yields even larger relative gains when coupled with adaptive sampling. The resulting predictions are more consistent with the governing PDE. Physical structure, reduced-basis interpretability, and integrated mobile sensor control together enable accurate source recovery under sparse and noisy data. These results demonstrate the effectiveness of combining structure-preserving modeling with adaptive sampling for inverse problems in transport-dominated systems.
\section{Problem statement}
\label{sec:problem_statement}
We study inverse source identification problems governed by steady-state advection–diffusion physics on a bounded domain $\Omega \subset \mathbb{R}^d$ with a given divergence-free velocity field and P\'eclet number. The scalar field $u : \Omega \to \mathbb{R}$ satisfies the forward model used in experiments,
\begin{equation}
    \label{eq:advec-diff}
    - \frac{1}{\text{Pe}} \nabla^2 u + \mathbf{v} \cdot \nabla u = f
    \quad \text{in } \Omega,
    \qquad
    u|_{\partial \Omega} = g,
\end{equation}
where $\text{Pe} > 0$ is the P\'eclet number,
$\mathbf{v}: \Omega \to \mathbb{R}^d$
is an unknown, divergence-free velocity field scaled so that \(\|\mathbf v\|_{L^\infty(\Omega)}=1\), $g$ is a fixed boundary condition, and $f:\Omega \to \mathbb{R}$ is an unknown source field.

We assume access to a set of $N$ mobile sensors at locations $X = \{\mathbf{x}_i\}_{i=1}^N$, each recording local measurements of the scalar and velocity fields, along with the global P\'eclet number:
\begin{equation}
    \label{eq:z}
    z_i = \left( \mathbf{x}_i,\; u(\mathbf{x}_i) + \eta_i^{(u)},\; \mathbf{v}(\mathbf{x}_i) + \eta_i^{(v)}, \text{Pe} \right),
\end{equation}
\noindent where $\eta_i^{(u)}$, $\eta_i^{(v)}$ are independent additive noise terms. We aim to infer the unknown source field $f$ from these noisy, sparse observations, and to guide sensor placement toward regions that improve reconstruction quality \cite{tarantola2005inverse}.

In this work we must solve two subproblems (SP1-2):
\begin{hypobox}
\begin{SPenum}
    \item \label{sp:SP1} \textbf{Source identification:} Given sparse, local observations $z = \{z_i\}$, estimate a normalized source density $\rho_\theta(\cdot \mid z)$ from a parameterized model, such that the corresponding scalar field $\tilde{u}_\theta$, obtained by solving \eqref{eq:advec-diff} with learned source $f_\theta $, matches the supervising field and source data.
    \item \label{sp:SP2} \textbf{Adaptive sensor placement:} Use the predicted density $\rho_\theta(\cdot \mid z)$ as an importance function to adaptively steer the sensor positions $X$ by minimizing the coverage energy functional:
    \begin{equation}
        \label{eq:generic_J}
        \mathcal{J}(X) = \int_\Omega \rho_\theta(\mathbf{x} \mid z)\, \min_{1 \le i \le N} g(d(\mathbf{x}, \mathbf{x}_i)) \, d\mathbf{x},
    \end{equation}
    where $d(\mathbf{x}, \mathbf{y})$ is a distance metric on the domain and $g(\cdot)$ is a non-decreasing function describing sensor degradation \cite{cortes2004coverage}.
\end{SPenum}
\end{hypobox}


\noindent In concert, these form a closed-loop inference-and-control problem: we alternate between estimating the source from current observations and repositioning sensors to improve future inference. The remainder of the paper develops a structure-preserving surrogate model that drives this loop, along with an implementation of a geodesic Lloyd algorithm \cite{lloyd1982least, cortes2004coverage} for sensor coverage optimization. We demonstrate the combination for an adaptive sensor-placement algorithm for source localization with supporting analysis and empirical convergence to the true source location and improvement over baseline models without structure preservation.
\section{Prior work}
\label{sec:background}

\subsection{Inverse source identification in PDEs}
\label{sec:background_si}

Classical approaches to source identification in transport systems formulate the problem as a PDE-constrained optimization task, typically regularized via Tikhonov penalties and solved using adjoint-state or iterative methods \cite{chavent2010nonlinear, bakushinsky2005iterative}. These approaches offer interpretability and explicit control over regularization and error metrics, but they can degrade under sparse/noisy observations and require repeated PDE solves, which is challenging for mobile or real-time sensing. Alternative strategies based on Bayesian inversion \cite{tarantola2005inverse, stuart2010inverse} and model-based data assimilation \cite{alpay2000model} have been developed to mitigate these limitations.

\subsection{Sensor placement and source-seeking strategies}
\label{sec:background_sensor_placement}
When sensing is expensive or constrained, accurate inversion depends not only on the quality of the model but on data collection itself \cite{matthes2005source, weimer2009multiple}. In robotics and environmental monitoring, heuristic strategies like chemotaxis \cite{russell2003comparison, marques2002olfaction}, infotaxis, and entropy maximization have long been used for source localization \cite{russell2003comparison, webster2012bioinspired, dhariwal2004bacterium, zarzhitsky2005distributed, soares2015distributed}.
Gaussian Process models provide a principled Bayesian approach and can be integrated with prior physical data \cite{khodayi2019model}.
Complementing these are coverage-based methods such as Lloyd’s algorithm and its generalizations, which provide a principled approach by minimizing the coverage cost in \eqref{eq:generic_J} \cite{hajieghrary2016multi, hajieghrary2017information}.

\subsection{Data-driven inverse modeling}
\label{sec:background_inverse}

Machine learning methods have emerged as a flexible alternative for solving inverse problems under sparse and noisy data, or in uncertain settings \cite{de2022deep, mishra2022estimates, kamyab2022deep, vesselinov2018contaminant, kontos2022machine, reiter2017machine}. 
Operator learning methods, such as DeepONets and Fourier neural operators, attempt to learn mappings between function spaces directly \cite{li2020fourier, lu2021learning}, and can be used for inverse problems where the input is the solution space \cite{molinaro2023neural}.
Physics-informed neural networks (PINNs) incorporate differential constraints into the loss function \cite{raissi2019physics} and have been widely explored for inverse problems \cite{yu2022gradient, jagtap2022physics}.
Broadly, physical structure is imposed weakly if at all, and the coupling between the sampled field and the source term of interest must be implicitly deduced from training \cite{karniadakis2021physics, wang2023physics, cai2021physics}.
The resulting models can be unstable or uninterpretable \cite{wang2021understanding, jiang2024structure}.

\subsection{Structure-preserving learning}
\label{sec:background_ml}
A promising direction in machine learning is the development of \emph{structure\hyph preserving machine learning} strategies which incorporate physical or geometric structure either by construction of neural architechtures or by training with hard constraints, enforcing e.g: symmetries, conservation laws, or geometric invariants \cite{greydanus2019hamiltonian, cohen2016group, brandstetter2022clifford, chen2021neural, hernandez2021structure, patel2022thermodynamically}.
We consider methods grounded in geometric discretization, where structure preservation is enforced through differential form-based representations of physical quantities through Finite Element Exterior Calculus (FEEC) \cite{arnold2006finite}.

Discrete Exterior Calculus (DEC) and Finite Element Exterior Calculus (FEEC) provide tools for discretizations with exact \emph{discrete} preservation of topological identities (e.g., Stokes/Gauss), flux balances, and cohomology classes~\cite{arnold2018finite, desbrun2005discrete}. Whitney forms, in particular, define basis functions that represent fields, fluxes, and sources in conforming mixed finite element spaces that interpolate corresponding differential forms, inheriting in the process exact treatment of the generalized Stokes theorems \cite{lohi2021whitney}.
These foundations have been extended to machine learning, with recent methods incorporating data-driven Whitney forms, learned partitions of unity, and graph-based exterior calculus into neural architectures \cite{trask2022enforcing, actor2024data, jiang2024structure}. The resulting models offer conservation by design, interpretable structure, and principled integration with mesh-based physical domains, making them especially well-suited to inverse problems where both the source and its observational effects are governed by conservation laws.


\section{Methodology}
\label{sec:methodology}

\subsection{Learnable structure-preserving reduced order model}
\label{sec:methodology_pde_surrogate}

We summarize the approach developed in \cite{kinch2025structure} here, but defer to the previous work for details. Our contribution at the ROM level is the introduction of a conditional source field, in addition to the conditional shape functions and physics considered in the original work.

The original governing equations \eqref{eq:advec-diff} hold at the differential level; however, after an aggressive dimension reduction, we assume that only the conservation structure holds
\begin{equation}
    \label{eq:conservation}
    \nabla \cdot \mathbf{F}(u) = f \quad \text{in } \Omega,
\end{equation}
where \( \mathbf{F} \) is a conserved flux and \( f \) is a scalar source field; the fluxes encoded by $\mathbf{F}$ are assumed to contain the effective homogenization of the physics onto a reduced basis and therefore differ from the original model. In the data-driven FEEC framework \cite{kinch2025structure, actor2024data, trask2022enforcing}, the unknown fluxes are posed as a Lipschitz nonlinearity (potentially conditioned on $z$) stabilized by a diffusive flux, 
\begin{equation}
    \mathbf{F}_\theta(u, z) = -\varepsilon \nabla u + \mathcal{N}_\theta(u, z),
\end{equation}
for some trainable \( \varepsilon > 0 \). Discretizing with a mixed Galerkin setting taking $u$ and $\mathbf{F}$ as Whitney forms, we obtain a discrete nonlinear model encoding the conservation relationships. Identifying the mass matrices of Whitney forms with the Hodge star in the Discrete Exterior Calculus (DEC), we can interpret the flux as a parameterized co-chain to co-chain map between degrees of freedom for $u$ and $\mathbf{F}$; put plainly, the Whitney forms describing $u$ and $\mathbf{F}$ model geometry associated with control volumes and their areas, respectively, while $\mathcal{N}$ models the exchange of conserved quantities between them.

This approach offers particular advantages over black box techniques:
\begin{itemize}[leftmargin=*]
    \item It incorporates complex geometry directly through the finite element bilinear forms, avoiding issues encountered by popular neural operators on non-trivial domains.
    \item It guarantees exact conservation of flux–source balance at the discrete level by construction so that predictions satisfy conservation principles independent of data sparsity.
    \item The model admits a Lax-Milgram-style stability theory, with a consequent well-posedness theorem (See \cite[Thm. 2.1]{kinch2025structure}, \cite[Thms. 3.1-5]{trask2022enforcing}) that guarantees existence of weak solution even when the model is solved out of distribution. This imposes an inductive bias toward regular, physically-meaningful solutions in data sparse regimes. 
    \item It explicitly couples the source prediction to the learned flux through the local conservation principle (See e.g.\cite[Sec. 3]{trask2022enforcing}), exposing an exploitable linkage between local sensor readings of the field/flux state and the global location and magnitude of sources.
\end{itemize}

By conditioning $\mathcal{N}$ and $f$ on sensor readings, we obtain a parametric discrete model which specializes to sensor data through the conditioning variable $z$; after denoting discrete operators with a superscript $h$, we express the model as
\begin{equation}
    \label{eq:learned-pde}
    -\varepsilon \Delta^h u + \nabla^h \cdot \mathcal{N}^h_\theta(u, z) = \mathcal{S}^h_\theta(z) \quad \text{in } \Omega, \quad u|_{\partial \Omega} = g.
\end{equation}
Further details on the discrete system may be found in Section~\ref{sec:methodology_discretization}. While the conditioning of digital twins through trainable Whitney forms and fluxes was first described in \cite{kinch2025structure}, the current work introduces the learned source term \( \mathcal{S}_\theta(z) \), which is crucial for the task of source identification. After performing the conditioning step in \textbf{SP1}, the learned $\mathcal{S}^h_\theta(z)$ is then used to construct the importance function $\rho_\theta(\mathbf{x} \mid z)$ in \textbf{SP2}.

\subsection{Discretization}
\label{sec:methodology_discretization}

We summarize the discretization process here, referring readers to Trask et al. \cite{trask2022enforcing} for a detailed description of data-driven Whitney forms and Kinch et al. \cite{kinch2025structure} for further details of the conditioning transformer and problem setup. 

The scalar field $u$ and source $f$ are identified as 0-forms, while the flux $\mathbf{F} $ is identified as a 1-form; both are expanded in low order Whitney forms of the corresponding degree (specifically, $\mathcal{P}_1$ and Nedelec elements) on a simplicial mesh of $\Omega$. Let $\{ \varphi_i^0 \}_{i=1}^{n_0^f}$ denote the 0-form basis with nodal degrees of freedom, and $\{ \varphi_e^1 \}_{e=1}^{n_1^f}$ the 1-forms with tangential edge degrees of freedom.

To reduce computational cost and allow for adaptive resolution, we construct a data-dependent reduced basis for 0-forms. Specifically, we define coarse scalar basis functions $\{ \psi_j^0 \}_{j=1}^{n_0^c} \subset H^1(\Omega)$ as convex combinations of the fine-scale basis:
\begin{equation}
    \psi_j^0(x) = \sum_{i=1}^{n_0^f} W_{ji}(z)\, \varphi_i^0(x),
\end{equation}
where $W(z) \in \mathbb{R}^{n_0^c \times n_0^f}$ is a learned, nonnegative matrix with:
\begin{equation}
    W_{ji} \geq 0, \quad \sum_{j=1}^{n_0^c} W_{ji} = 1 \quad \text{for all } i.
\end{equation}
This construction ensures that $\{ \psi_j^0 \}$ forms a partition of unity over $\Omega$, and defines a set of interpretable control volumes. This exposes a fully connected graph structure, with the learned fluxes operating on the edges. In practice this basis may be of order 1-10 elements, drastically compressing the original $\mathcal{P}_1$ basis.

Following the standard FEEC construction, coarsened 1-forms must be constructed which support a surjective coboundary operator. We do this by applying the standard construction of Nedelec elements to the coarsened basis.
\begin{equation}
    \psi^1_{ij} = \psi^0_j \nabla \psi^0_i - \psi^0_i \nabla \psi^0_j,
\end{equation}
defining subspaces of the original de Rham complex which form a de Rham complex of their own. The current work requires only the $div-grad$ subcomplex; see \cite{actor2024data} for details of the complete complex.

All quantities are represented on this reduced basis without loss of structure. The reduced scalar field coefficients $\hat{u} \in \mathbb{R}^{n_0^c}$, fluxes $\hat{\mathbf{F}} \in \mathbb{R}^{n_1^c}$, and sources $\hat{f} \in \mathbb{R}^{n_0^c}$ are defined in terms of coarse Whitney forms $\{ \psi_j^0 \}, \{ \psi_{ij}^1 \}$. We may easily pull back the coarsened Whitney forms onto the fine scale mesh through the affine operations \(f_\theta=W^\top\hat{f}_\theta\), and \(u_\theta=W^\top\hat{u}_\theta\).

The discrete conservation law is expressed in the reduced basis as:
\begin{equation}
    \label{eq:discrete_conservation}
    \underbrace{\varepsilon\delta_0^\top M_1 \delta_0 \hat{u} + \delta_0^\top M_1 \mathcal{N}_{\theta}(\hat{u}, z)}_{\mathcal{F}(\hat{u}; \theta, z)}
    =
    \underbrace{M_0 \hat{f}_\theta(z)}_{b(\theta, z)},
\end{equation}
where $M_0, M_1$ are reduced mass matrices for the coarsened 0- and 1- forms, respectively, and $\delta_0$ is the reduced coboundary operator. This yields the PDE constraint,
\begin{equation}
    H(\hat{u}; \theta, z) := \mathcal{F}(\hat{u}; \theta, z) - b(\theta, z) = 0.
\end{equation}
See~\cite{kinch2025structure, trask2022enforcing} for details of matrix assembly and solver setup.

\subsection{Learnable parameterization}
We parameterize the source, flux, and basis networks in \eqref{eq:discrete_conservation} as task-specific heads built on a shared transformer encoder, \(\mathcal{E}_\theta\). The encoder maps the unordered sensor inputs to a latent representation, which conditions each component of the surrogate PDE model (adaptive basis, flux, and source). The implementation is based on the cross-attention transformer CNWF method proposed in \cite{kinch2025structure}, with modifications to accommodate the adaptive sensing and source identification application. Full architectural details are given in Appendix~\ref{app:architecture}.

\subsection{Source identification}
\label{sec:methodology_si}

In the source identification problem, we seek to recover the unobserved source field responsible for generating sparse and noisy observations of the scalar field \( u \), under the governing PDE (advection-diffusion in our experiments). We assume access to a set of \( N \) sensors distributed across the domain \( \Omega \), which make point measurements to form the model conditioning input, \( z = \{ z_i \}_{i=1}^N \). 

Notably, the underlying source field \( f \) is unobserved. Recovering it from indirect field measurements requires solving an ill-posed inverse problem, where physical constraints are essential for meaningful inference.

The predicted source $f_\theta(z)$ is normalized to integrate to one (after projection onto \(\varphi^0\)), with nonegativity enforced (e.g. via ReLU),
\begin{equation}
    \label{eq:density}
    \rho_\theta(\mathbf{x}\,|\,z)
    \;=\;
    \frac{f_\theta(z)(\mathbf{x})}
         {\displaystyle\int_{\Omega}f_\theta(z)(\mathbf{y})\,\mathrm{d}\mathbf{y}},
    \qquad
    \rho_\theta\ge 0,
    \;
    \int_{\Omega}\rho_\theta=1,
\end{equation}
so that $\rho_\theta$ may be interpreted as a spatial probability density for the point-source location, or source distribution in complex source cases. 

To train the model, we solve the PDE-constrained optimization problem for a given normalized source term \(\rho_{\mathrm{true}}\) and the corresponding solution \(u_{\mathrm{true}}\) as:
\begin{align}\label{eq:opt}
    \min_{\theta} \; &\| u_\theta(\theta, z) - u_{\text{true}} \|^2 + W_{2,\epsilon}^2(\rho_{\mathrm{true}},\rho_\theta(\,\cdot\,|\,z))
    \\
    &\text{subject to} \quad H(u_\theta, \theta, z) = 0,
\end{align}
where $W_{2,\epsilon}^2$ corresponds to the 2-Wasserstein distance as an optimal transport metric for the source reconstruction; due to the compact nature of the source, a simple MSE objective performed poorly.
We approximately compute the geodesic Wasserstein metric using the sinkhorn distance \cite{cuturi2013sinkhorn}.
\begin{itemize}[leftmargin=*]
    \item Following the well-posedness of Equation \ref{eq:discrete_conservation}, \ref{eq:opt} has a non-empty feasible set, and thus any solution satisfies Equation \ref{eq:discrete_conservation}.
    \item Consequently, the learned model enforces topological structure independent of data availability, optimizer error, or extant of extrapolation.
\end{itemize}
See \cite{kinch2025structure} for details on the solution of the constrained optimization problem.

\subsection{Adaptive source localization} 
\label{sec:methodology_lloyds}
Given the model-defined source distribution $\rho_\theta$, we implement the classic Lloyd's algorithm for optimal coverage to distribute sensors within the domain \cite{cortes2004coverage}. We describe a geodesic implementation to allow generalization to complex, non-convex geometries, as shown in the results. Lloyd's algorithm arises from quantization and, in the Euclidean setting, provides monotonic descent towards a weighted centroidal Voronoi tessellation (CVT) \cite{du1999centroidal}, which is known to minimize the coverage functional (e.g. \eqref{eq:generic_J}), and therefore has been adapted in distributed sensor placement \cite{cortes2004coverage}. For all computations we rely on a discretized approximation on the original mesh, allowing for accurate evaluation of integrals with quadrature and computational efficiency.

We start by defining the geodesic distance inside the possibly non-convex domain~$\Omega$ by $d_\Omega(\mathbf{x},\mathbf{y})$,

\[
    \label{eq:geodesic}
    d_\Omega(\mathbf{x},\mathbf{y})
    =
    \inf_{\substack{\gamma\in C^{1}\!\bigl([0,1],\overline{\Omega}\bigr)\\
                    \gamma(0)=\mathbf{x},\,\gamma(1)=\mathbf{y}}}
    \int_{0}^{1}\!
    \bigl\|\gamma'(t)\bigr\|\,\mathrm{d}t,
\]

\noindent i.e. the length of the shortest $C^1$ path contained in
$\overline{\Omega}$ that joins the points
$\mathbf{x},\mathbf{y}\in\overline{\Omega}$, where $\overline{\Omega}=\Omega \cup \partial\Omega$~\cite{kimmel1998computing}. Let
\[
    V_i(\mathbf{X})
    =\bigl\{\mathbf{x}\in\Omega:\,d_\Omega(\mathbf{x},\mathbf{x}_i)
    \le d_\Omega(\mathbf{x},\mathbf{x}_j),\;j\neq i\bigr\}
\]
denote the geodesic Voronoi cell for a generator $x_i$.

By selecting the sensor positions
\(\{x_i\}\) as the Voronoi generators and using the geodesic distance function and \(g(x):=x^2\), the localization objective \eqref{eq:generic_J} is:
\begin{align}
    \label{eq:coverage_actual}
    \mathcal{J}(\mathbf{X})
    &=
    \sum_{i=1}^N
    \int_{V_i(\mathbf{X})}
    \rho(\mathbf{q})
    d_\Omega^2(\mathbf{q},\mathbf{x}_i)
    \mathrm{d}\mathbf{q}, 
\end{align}
i.e., the expected squared geodesic distance from a random point within the domain to the nearest sensor, weighted by the importance, \(\rho\). 

In the original paper \cite{cortes2004coverage}, $\rho(\mathbf{q})$ is an \emph{a priori} specified field, typically empirically via heuristics. In contrast, for the current work $\rho$ is prescribed via $\mathcal{S}^h_\theta(z)$ so that the localization objective weight is constrained to satisfy a physics-based PDE. Unlike the original work, $\rho$ is now an implicit function of the sensor locations, and so the original analysis in \cite{cortes2004coverage} no longer holds. We present analysis following from this modification in Sections \ref{sec:methodology_convergence}, \ref{sec:methodology_model_based}, and \ref{sec:methods_source_conv}.

The aim of our sensor placement approach is to minimize \eqref{eq:coverage_actual} over a given, unknown source term \(\rho_{\mathrm{true}}\).
In the case of a single event as the source term with a probability distribution \(\rho_{true}\), the localization objective corresponds to maximizing the observation probability over the domain for sensors whose sensing capabilities degrade (e.g. due to noise) as \(d_\Omega(\cdot)^2\). When \(\rho \) corresponds to the location of an emitter, this can be viewed as improving the observability of the source, and therefore is expected to improve the conditioning of the inverse problem.

To update the sensor locations in the conventional Lloyd's implementation we define a gradient flow towards the centroid, \(c_i\), with a relaxation \( \alpha \in (0,1] \), 
\begin{equation}
    \label{eq:cont_lloyd_update}
    \dot{x}_i = -\alpha(x_i-c_i).
\end{equation}
To modify this for discrete updates the nonconvex setting and enable coverage on complex geometries, we consider the generalized geodesic path, $\gamma$, between the sensor and the centroid.
We define the update as
\begin{equation}
    \label{eq:relaxed-update}
    \mathbf{x}_i^{(k+1)} := \gamma_{\mathbf{x}^k_i\to \mathbf{c}^k_i}^{(k)}(\alpha),
\end{equation}
i.e. a movement along the geodesic path proportional to $\alpha$. This reduces to \eqref{eq:cont_lloyd_update} in the convex setting.
\noindent This formulation preserves feasibility by construction since \( \gamma_i^{(k)}(t) \in V_i \subset \Omega \) for all \( t \in [0,1] \), and provides a tunable tradeoff between convergence speed and robustness in the presence of uncertainty. In practice, the path \( \gamma_i^{(k)} \) is discretized using fast marching and interpolated using geodesic barycentric coordinates on the mesh.

The geodesic centroid of a Voronoi region \( V_i \subset \Omega \) is defined as the intrinsic mean:
\begin{equation}
  \label{eq:mean}
  \mathbf{c}_{i}
  =
  \arg\min_{\mathbf{y} \in \Omega}
  \int_{V_{i}}
    d_{\Omega}(\mathbf{y},\mathbf{x})^{2}\rho_{\theta}(\mathbf{x})
    \,\mathrm{d}\mathbf{x},
\end{equation}
which is guaranteed to be unique on Hadamard manifolds~\cite{afsari2011riemannian}. In Euclidean space, this reduces to the standard centroid:
\[
    \label{eq:centroid}
    \mathbf{c}_i = \frac{1}{M_i} \int_{V_i} \mathbf{x}\, \rho_\theta(\mathbf{x})\, \mathrm{d}\mathbf{x}, \quad M_i = \int_{V_i} \rho_\theta(\mathbf{x})\;\mathrm{d}\mathbf{x}.
\]

Since computing \eqref{eq:mean} is expensive, we approximate it by computing the Euclidean centroid \( \bar{c}_i \) and projecting it into the geodesic Voronoi cell if it lies outside the region:
\begin{equation}
    \label{eq:projected-centroid}
    \mathbf{c}_i := \arg\min_{\mathbf{x} \in V_i} \; \| \mathbf{x} - \bar{\mathbf{c}}_i \|_2^2.
\end{equation}
This ensures feasibility and, in convex geometries, reproduces the classic Lloyd update. Algorithm~\ref{alg:inner} describes the discrete, geodesic Lloyd algorithm for a given set of sensor locations and a density.
Geodesic distances are computed using the exact discrete geodesic algorithm of Mitchell–Mount–Papadimitriou as implemented in \texttt{pygeodesic}~\cite{mitchell1987discrete}; integrals are evaluated by quadrature over the mesh.

\begin{algorithm}[H]
\caption{\textsc{DiscreteLloyd}\,$(X,\rho,M,\alpha)$}
\label{alg:inner}
\begin{algorithmic}[1]
\State \textbf{Input:}
       sensor set $X=\{x_i\}$,\; density $\rho$,\; iterations $m$
\For{$k = 1,\dots,m$}
    \State Construct geodesic Voronoi diagram
           $\{V_i\}$ of $X$ in $\Omega$
    \For{$i = 1,\dots,N$ \textbf{in parallel}}
        \State Compute Euclidean centroid
               $\bar{\mathbf{c}}_i\gets\frac{1}{M_i}
               \displaystyle\int_{V_i}\!\mathbf{x}\,\rho(\mathbf{x})\,d\mathbf{x}$
        \State Project onto feasibility (if $\Omega$ non-convex)
                $\mathbf{c}_i \gets \arg\min_{\mathbf{x} \in V_i} \; \| \mathbf{x} - \bar{\mathbf{c}}_i \|_2^2.$
        \State Compute geodesic path
                $\gamma_i$ from $\mathbf{x}_i$ to $\mathbf{c}_i$
        \State Update
                $\mathbf{x}_i^{\,\text{new}}
                \gets
                \gamma_i^{(k)}(\alpha)$
    \EndFor
\EndFor
\State \textbf{Return} updated set of sensor positions $X$
\end{algorithmic}
\end{algorithm}

\subsection{Convergence of source localization algorithms}
\label{sec:methodology_convergence}

We now examine the convergence of Lloyd’s update rule when the importance field varies externally over time. Lloyd's algorithm is known to produce a monotonic energy descent of the coverage functional for a convex geometry with a fixed importance, guaranteeing convergence to a CVT \cite{cortes2004coverage}. We must consider the more complex setting where the importance density is a time varying function of the sensor locations, as is the case in our application. We establish a sufficient condition under which convergence is preserved despite variations.

We define the Lloyd update rule for a set of generator positions $x_k$ and importance field $\rho_k$ as $G(x_k, \rho_k) =: x_{k+1}$, and note that in this general setting the algorithm implemented in $G$ can be $m$ steps of the the discrete Lloyd described in Algorithm \ref{alg:inner}, or some time integration in a continuous implementation. We consider the corresponding set of importance updates $\{\rho_k\}_{k=1}^K$ coming from a model $\rho_{k+1}:=\rho_\theta(x_{k})$ with an initial estimate $\rho_0$, and the resulting set generator positions be $\{x_k = \{x_1^k, \ldots, x_n^k\}\}_{k=1}^K$. Let $\mathcal{J}_{\rho}(x)$ denote the coverage energy functional with importance $\rho$ and generator locations $x$.

We assume throughout this section that the current sensor configuration \( x_k \) is suboptimal with respect to the current importance field \( \rho_k \) (i.e., not a centroidal Voronoi configuration). This ensures that the Lloyd update yields strict energy descent.
We define the strictly decreasing change in the coverage energy $\mathcal{J}$ due to the Lloyd update as:
\[
    \Delta \mathcal{J}_k^{\mathrm{Lloyd}} := \mathcal{J}_{\rho_k}(x_{k+1}) - \mathcal{J}_{\rho_k}(x_{k}) < 0,
\]
additionally we define the change in the coverage energy due to the model update
\[
    \Delta \mathcal{J}_k^{\mathrm{model}} := \mathcal{J}_{\rho_{k+1}}(x_{k+1}) - \mathcal{J}_{\rho_k}(x_{k+1}).
\]

\begin{theorem}[Monotonic decrease of model-relative coverage]
\label{thm:1}
We consider the general setting with no assumption of the dependence of $\rho_k$ on $x_k$ for this theorem. 
If the importance update satisfies
\[
    \label{eq:theorem1_condition}
    \|\rho_{k+1} - \rho_k\|_{L^\infty} < \frac{-\Delta \mathcal{J}_k^{\mathrm{Lloyd}}}{C_\Omega},
\]
for some constant \( C_\Omega > 0 \) depending only on the domain geometry, then the combined two-step update satisfies
\[
    \mathcal{J}_{\rho_{k+1}}(x_{k+1}) < \mathcal{J}_{\rho_k}(x_k),
\]
and the model-relative coverage energy is strictly decreasing.
\end{theorem}

\begin{proof}
The total energy change across one iteration is:

\begin{align*}
    \mathcal{J}_{\rho_{k+1}}(x_{k+1}) - \mathcal{J}_{\rho_k}(x_k)
    &= \mathcal{J}_{\rho_{k+1}}(x_{k+1}) - \mathcal{J}_{\rho_k}(x_{k+1})
    + \mathcal{J}_{\rho_k}(x_{k+1}) - \mathcal{J}_{\rho_k}(x_k)
    \\
    &= \Delta\mathcal{J}_k^{\mathrm{model}} + \Delta\mathcal{J}_k^{\mathrm{Lloyd}}.
\end{align*} 

To bound the model update term $\Delta\mathcal{J}_k^{\mathrm{model}}$, we use the linearity of $\mathcal{J}$ in $\rho$:
\begin{align*}
    |\Delta\mathcal{J}_k^{\mathrm{model}}| 
    &= \left| \sum_{i=1}^n \int_{V_i(x_{k+1})} \left( \rho_{k+1}(q) - \rho_k(q) \right) d_\Omega(q, x_i^{k+1})^2 \, dq \right|
    \\
    &\leq \|\rho_{k+1} - \rho_k\|_{L^\infty} \sum_{i=1}^n \int_{V_i(x_{k+1})} d_\Omega(q, x_i^{k+1})^2 \, dq.
\end{align*}

Define the geometry-dependent constant
\[
    C_\Omega := \sup_{x \in \Omega^n} \sum_{i=1}^n \int_{V_i(x)} d_\Omega(q, x_i)^2 \, dq,
\]
which is finite on $\Omega$ with bounded diameter. 
Then,
\[
    \Delta\mathcal{J}_k^{\mathrm{model}} \leq |\Delta\mathcal{J}_k^{\mathrm{model}}| \leq C_\Omega \|\rho_{k+1} - \rho_k\|_{L^\infty}.
\]

Combining the two terms, the total energy change satisfies:
\[
    \mathcal{J}_{\rho_{k+1}}(x_{k+1}) - \mathcal{J}_{\rho_k}(x_k)
    \leq C_\Omega \|\rho_{k+1} - \rho_k\|_{L^\infty} + \Delta\mathcal{J}_k^{\mathrm{Lloyd}}.
\]

By the assumption of the theorem, we have:
\[
    \|\rho_{k+1} - \rho_k\|_{L^\infty} < \frac{-\Delta\mathcal{J}_k^{\mathrm{Lloyd}}}{C_\Omega},
\]
which implies that the right-hand side is strictly negative due to the strict Lloyd energy descent. Therefore,
\[
    \mathcal{J}_{\rho_{k+1}}(x_{k+1}) < \mathcal{J}_{\rho_k}(x_k),
\]
and the model-relative coverage energy is strictly decreasing.
\end{proof}

\paragraph{Remark 1} 
We view $\Delta J^{\mathrm{Lloyd}}$ as a buffer which increases the tolerance to changes in the importance density, $\rho$.
The isolation of the sensor movement from the importance updates gives an elegant opportunity to control this buffer by simply increasing the number of inner Lloyd steps taken.
For a given proposed importance update $\rho_{new}$ we can guarantee the condition in \eqref{eq:theorem1_condition} to be satisfied by increasing $m$, so long as the sub-optimality of the initial configuration is greater than the change in energy due to the importance update, i.e.
\[
    \|\rho_{k+1} - \rho_k\|_{L^\infty} < \frac{\mathcal{J}_{\rho_k}(x_{k}) - \mathcal{J}_{\rho_k}(x^*)}{C_\Omega},
\]
where $x^*$ are the optimal generator positions, \(x^*:=\text{argmin}_x \mathcal{J}_{\rho_k}(x)\).

\paragraph{Remark 2}
\label{rem:gurantee_imp_regularity}
We note that we can alternatively ensure the regularity of the importance updates and therefore guarantee Theorem 1 is satisfied if we employ an optimized convex update to the importance field given a candidate update $\rho_{\mathrm{new}}$ as:
\[
    \rho_{k+1} = \alpha \rho_k + (1-\alpha) \rho_{\mathrm{new}}
\]
for
\[
    \alpha = \inf\left\{\alpha' \in [0,1] | \; \rho_{k+1}\; \text{satisfies \eqref{eq:theorem1_condition}}\right\}.
\]
This provides a guaranteed valid update in all $\rho_{\mathrm{new}}$ for $\alpha=0 \implies \norm{\rho_{k+1} - \rho_k} = 0$.

\subsection{Model driven importance updates}
\label{sec:methodology_model_based}

We now consider the case where the importance field is a function of the sensor locations, as in our application, where \(\rho(x)\) is parameterized with the CNWF source network on the learned basis functions.

\begin{theorem}
[Monotonic decrease of true coverage objective]
\label{thm:2}
Let: $\rho(x)$ be a map from sensor positions to an estimated importance field,
$\mathcal{J}_{\rho}(x)$ denote the model-relative coverage functional, and $\mathcal{J}_{\rho_{\mathrm{true}}}(x)$ be the true coverage functional. Recall from Theorem \ref{thm:1} that $\mathcal{J}$ is Lipschitz in $\rho$ with constant $C_\Omega$.

If the condition
\[
     \Delta\mathcal{J}_k^{\mathrm{Lloyd}} < 
     -2C_\Omega \|\rho(x_k) - \rho_{\mathrm{true}}\|_\infty
\]
is satisfied, and the generator position updates follow a Lloyd descent,
then the true coverage energy strictly decreases:
\[
    \mathcal{J}_{\rho_{\mathrm{true}}}(x_{k+1}) < \mathcal{J}_{\rho_{\mathrm{true}}}(x_k).
\]
\end{theorem}

\begin{proof}
We decompose the change in true coverage energy into model errors at steps $k$, $k+1$ and the descent step:
\begin{align*}
    \Delta\mathcal{J}_k=&\mathcal{J}_{\rho_{\mathrm{true}}}(x_{k+1}) - \mathcal{J}_{\rho_{\mathrm{true}}}(x_k)
    \\
    =& \left[\mathcal{J}_{\rho_{\mathrm{true}}}(x_{k+1}) - \mathcal{J}_{\rho_k}(x_{k+1})\right]
    + \Delta \mathcal{J}_k^{\mathrm{Lloyd}}
    + \left[\mathcal{J}_{\rho_k}(x_k) - \mathcal{J}_{\rho_{\mathrm{true}}}(x_k)\right].
\end{align*}

Applying the Lipschitz continuity in $\rho$:
\[
    |\mathcal{J}_{\rho_{\mathrm{true}}}(x_k) - \mathcal{J}_{\rho_k}(x_k)| \leq C_\Omega \|\rho_k - \rho_{\mathrm{true}}\|_\infty,
\]
\[
    |\mathcal{J}_{\rho_{\mathrm{true}}}(x_{k+1}) - \mathcal{J}_{\rho_k}(x_{k+1})| \leq C_\Omega \|\rho_k - \rho_{\mathrm{true}}\|_\infty.
\]

Thus
\[
    \mathcal{J}_{\rho_{\mathrm{true}}}(x_{k+1}) - \mathcal{J}_{\rho_{\mathrm{true}}}(x_k)
    \leq \Delta \mathcal{J}_k^{\mathrm{Lloyd}} + 2C_\Omega \|\rho_k - \rho_{\mathrm{true}}\|_\infty.
\]

If the condition in the theorem is satisfied and the combined prediction error is small relative to the Lloyd's descent energy change,
\[
    2C_\Omega \|\rho_k - \rho_{\mathrm{true}}\|_\infty  < \Delta \mathcal{J}_k^{Lloyd},
\]
then
\[
    \mathcal{J}_{\rho_{\mathrm{true}}}(x_{k+1}) < \mathcal{J}_{\rho_{\mathrm{true}}}(x_k),
\]
and thus we achieve strict decrease in the true coverage energy.

\end{proof}

\begin{theorem}
\label{thm:3}
(Monotonic decrease in prediction error bound)
If the generator updates provide a strict decrease in the coverage energy (e.g., if Theorem 2 is satisfied),
\[
    \mathcal{J}_{\rho_{\mathrm{true}}}(x_{k+1}) < \mathcal{J}_{\rho_{\mathrm{true}}}(x_k),
\]
and we additionally assume that the prediction error is bounded by a positive constant multiple of the true coverage energy, $C_\Phi$, such that,
\[
    \|\rho(x_k) - \rho_{\mathrm{true}}\|_\infty \leq C_\Phi\mathcal{J}_{\rho_{\mathrm{true}}}(x_k).
\]
It follows that the model prediction error bound is strictly decreasing,
\[
    \|\rho(x_{k+1}) - \rho_{\mathrm{true}}\| \leq C_\Phi \mathcal{J}_{\rho_{\mathrm{true}}}(x_{k+1}) < C_\Phi \mathcal{J}_{\rho_{\mathrm{true}}}(x_k).
\]
\end{theorem}

\begin{proof}
    Assume there exists a
constant $C_\Phi > 0$, such that for all $x$,
\[
    \|\rho(x) - \rho_{\mathrm{true}}\| \leq C_\Phi \mathcal{J}_{\rho_{\mathrm{true}}}(x).
\]

Under the strict energy decrease in Theorem 2,
\[
    C_\Phi \mathcal{J}_{\rho_{\mathrm{true}}}(x_{k+1}) < C_\Phi \mathcal{J}_{\rho_{\mathrm{true}}}(x_k),
\]
and therefore
\[
    \|\rho(x_{k+1}) - \rho_{\mathrm{true}}\|_\infty \leq C_\Phi \mathcal{J}_{\rho_{\mathrm{true}}}(x_{k+1}) < C_\Phi \mathcal{J}_{\rho_{\mathrm{true}}}(x_k).
\]
\end{proof}

\paragraph{Remark 3} We interpret the constant \(C_\Phi\) as a relation between the informativeness of the sensor configuration, and the well-posedness of the inverse problem. For a sufficiently well-trained model, we expect this to correspond to improved predictions, or a reduced error bound. We note that this condition is dependent on model training and construction and cannot generally be guaranteed.

We conclude that if Theorems \ref{thm:1}, \ref{thm:2}, and \ref{thm:3} are satisfied, reducing the model-relative coverage energy results in both a reduction in the true coverage energy, and a strict improvement in model prediction error bound.
This establishes a positive feedback loop for the true coverage objective, where improvements in sensor positioning result in improved prediction error bounds, which may in turn result in improved sensor placement relative to the true coverage goal, provided the error bound, \(C_\Phi\), is sufficiently tight, which can be verified empirically.

In order for these to be satisfied we require that the importance updates satisfy a regularity condition over iterations, that the importance predictions satisfy an accuracy condition compared to the true field, and that the true coverage energy can be used to bound the prediction error. As described in Remark \ref{rem:gurantee_imp_regularity}, the regularity condition can be satisfied by construction, but the other conditions are dependent on model training and can not be easily guaranteed in practical settings, so we rely on empirical confirmation of the convergence of the prediction error, model-relative coverage, and true coverage energies to validate these assumptions are approximately satisfied in our experimental setting. 

\subsection{Convergence to source location}
\label{sec:methods_source_conv}
While the previous analysis established that coverage control with model updates can reduce the coverage cost, we now address the stronger property of convergence of a sensor to the true source location. 
We study a simplified but representative setting in which the source is a Dirac delta at $x^*$ and the learned model produces a smooth bump function centered near $x^*$ with a uniform background value, sharpening towards a Dirac as the nearest sensor approaches the source location. 
In this setting, we show that under continuous Lloyd dynamics the minimal distance between sensors and the source converges to zero exponentially fast, thereby closing the gap between coverage minimization and source convergence.

\begin{theorem} \label{thm:conv}
Let $\Omega \subset \mathbb R^2$ be a convex domain equipped with the Euclidean distance.  Let $x^* \in \Omega$ denote the true source location and 
\[
    m(t) := \min_{i} \|x_i(t) - x^*\|
\]
the minimal distance from any sensor $x_i$ to $x^*$. 
Suppose the learned importance function $\rho_\theta(\cdot \mid z)$ satisfies two conditions: 
\begin{enumerate}
    \item If $m(X) = 0$, then $\rho(X) = \delta(x^*)$.
    \item Centroids vary in a locally Lipschitz continuous manner that is contractive, meaning that there exists a critical distance $m_\mathrm{c}>0$ and $r \in (0,1)$ such that for all configurations $X_1$ and $X_2$ with respective minimal distances $m_1, m_2 < m_\mathrm{c}$, 
\end{enumerate}
\begin{equation}
        \|c_1(\rho(X_1))  - c_2(\rho(X_2)\| \le r|m_1 - m_2|. \label{lipbound}
    \end{equation}
    Then, under the continuous Lloyd dynamics in \eqref{eq:cont_lloyd_update},
    the minimal distance of sensors to $x^*$ converges to zero, with
    \[
        m(t) \le e^{\alpha(r-1)t} m(0).
    \]
\end{theorem}

\begin{proof}
    See Appendix \ref{app:conv_proof}.
\end{proof}

\emph{Remark.}
In Appendix \ref{app:conv_proof}, we also show a class of bump functions in which the convergence is global (holding under all initial configurations of sensors).  The  result above shows that under local Lipschitz continuity condition of the learned importance fields, Lloyd dynamics guarantees the convergence of sensors will converge to $x^*$. 
Continuity ensures stability of the centroid map and prevents spurious equilibria or oscillations; this highlights the need for importance fields that remain smooth and well-conditioned during closed-loop operation. This underscores the importance of a regular source prediction for the adaptive sensor placement and satisfies our second hypothesis~\ref{hyp:H2}.

\subsection{Two step optimization procedure}
\label{sec:methodology_two_step}

Motivated by the structure adopted in Theorem \ref{thm:1}, we adopt a two-scale optimization approach for minimizing the localization functional over the predicted importance fields. We consider the model updates to be a slowly varying external process, such that we can run an inner loop of Lloyd's updates with guaranteed convergence for a fixed importance, then we update the sensor measurements and model prediction.

Prior approaches to time‑varying coverage control \cite{lee2013controlled, lee2015multirobot, kennedy2019generalized, santos2019decentralized} design single\hyph loop control laws that guide agents toward centroids of externally defined, slowly evolving importance fields.
In contrast, our method employs a two‑scale update scheme: a fast inner Lloyd’s loop ensures descent of the model‑relative coverage cost for each fixed field, while a slower outer learner updates the importance function based on inference from sparse sensor observations. The density field $\rho$ is evolved via feedback from sensor placements and model predictions.

Explicitly, we define an adaptive sensor placement (outer) loop, wherein the sensor positions $X^{(k)}$ are fixed and the model predicted importance field $\rho_\theta(\cdot|z_k)$ is updated according to the sensor data in $z_k$. Then we fix the source field and run the discrete Lloyd's update loop for $m$ iterations.
This two scale optimization process is described in algorithm \ref{alg:outer}.

\begin{algorithm}[H]
\caption{Adaptive sensor placement with learned importance density}
\label{alg:outer}
\begin{algorithmic}[1]
\Require
    Initial sensor positions $X^{(0)}=\{x_i^{(0)}\}_{i=1}^{N}$,\;
    fixed CNWF parameters $\theta$,\;
    inner–loop length $m\ge1$, maximum outer iterations $K$
\vspace{.2em}
\For{$k = 0,1,\dots,K-1$}
    \State \textbf{Update data:}
           $z_k \;\gets\;
               \bigl\{u(x_i^{(k)}),\,\mathbf{v}(x_i^{(k)})\bigr\}_{i=1}^{N}$ 
    \State \textbf{Evaluate Source Model:}
           $f_k(x)\gets
            \mathcal{S}_\theta(z_k)$
    \State \textbf{Compute density:}
           $\rho_k(x)\gets
            \rho_\theta\!\bigl(x\,\big|\,z_k\bigr)$
           \Comment{\eqref{eq:density}}
    \State $X^{(k+1)} \gets
           \textsc{DiscreteLloyd}\!\bigl(X^{(k)},\rho_k,m\bigr)$
           \Comment{Algorithm\,\ref{alg:inner}}
\EndFor
\Ensure Final positions $X^{(K)}$
\end{algorithmic}
\end{algorithm}

\section{Numerical experiments}
\label{sec:exp}

\subsection{Data generation}
\label{sec:data_gen}
We define a domain $\Omega \subset \mathbb{R}^2$ discretized using a conforming triangular finite element mesh, $\phi^0$. Synthetic training data are generated by solving steady-state advection-diffusion equations of the form presented in \eqref{eq:advec-diff} by finite elements over the fine mesh discretization. We consider a compact, local source term parameterized by a bump function with a constant radius for each experiment,
\begin{equation}
    \beta_r(x)\;=\;
    \begin{cases}
        \displaystyle
        \exp\!\Bigl(-\dfrac{r^{2}}{\,r^{2}-\|x\|^{2}}\Bigr),
        & \|x\|<r,\\[1.2ex]
        0,
        & \|x\|\ge r.
    \end{cases}
\end{equation}
\noindent The bump function,  $\beta_r:\mathbb{R}^{d}\to[0,1]$ is $C^\infty$, equals $1$ at the origin, and has compact
support strictly inside the ball of radius~$r$.
For each training example, a random set of $N$ sensors is distributed across $\Omega$, returning noisy observations of the velocity and scalar field to form the full conditioning input, $z$. Together the solution, source, and conditioning inputs form the supervised training dataset, $\mathcal{D}:=\{z^{(k)}, u^{(k)}, f^{(k)}\}_{k=1}^{n_{samples}}$. Since all data is from FEM, we are able to quickly generate new samples on the fly to prevent overfitting. We specify an initial cache of $n_{samples}$ samples, and update a fixed number of randomly selected samples from this cache after every training iteration, effectively achieving an infinite dataset.

We restrict both source support and sensor placement to an interior subdomain. Let \( \delta > 0 \) be a user-defined exclusion distance, and define the valid interior domain
\[
    \Omega_{\mathrm{valid}} := \left\{ \mathbf{x} \in \Omega \;\middle|\; \operatorname{d}_\Omega(\mathbf{x}, \partial\Omega) \geq \delta \right\}.
\]
This constraint is also applied during the adaptive sampling.

Throughout, we impose a combination of homogeneous Dirichlet and Neumann boundary conditions on \( \partial\Omega \).
To ensure boundary conditions are strongly enforced on the reduced basis, we fix one of the trainable control volumes to represent the known boundary nodes and Dirichlet boundary conditions. Neumann boundary conditions are enforced implicitly.

We evaluate the proposed CNWF model across three benchmark scenarios of increasing geometric complexity: a simple circular domain with variable direction velocity, a regular mesh over the Gulf of Mexico with flow data coming from reanalysis simulations, and a maze geometry with simulated stokes flow following the path of the maze. These scenarios also correspond to different underlying hydrodynamic models, \(\mathcal{M}\) in \eqref{eq:ps1}, ranging from a nominal flow to a full physical model. In each case we used a fixed number of randomly distributed sensors, placed independently for each training and test sample.
Additionally, in each case we have a source of variability; the velocity direction for the circular geometry and the P\'eclet number for the Gulf and maze geometries. We note all cases have not been designed to match a particular physical scenario, but rather to convey the benefits of the CNWF localization method on challenging benchmarks.

\subsection{Baselines}
\label{sec:baselines}
To assess the benefit of structure preservation, we compare CNWF against two conventional black-box baselines: a standard multilayer perceptron (MLP) and an identical transformer model to the CNWF encoder with a linear layer to output, both trained to predict the source field directly from the sensor data with positive outputs from ReLU activations and normalized to a distribution following \eqref{eq:density}. The MLP provides a minimal complexity neural network parameterization, but as we will demonstrate struggles to capture the complex inverse process. The transformer baseline is identical to the encoder core of the CNWF model but lacks the head networks and structure preservation.

\section{Results}
\label{sec:results}

To verify the performance of the CNWF approach we first consider the inverse source identification problem alone then demonstrate its effectiveness when coupled with the adaptive sensor placement algorithm. We demonstrate accurate source reconstruction from limited, noisy measurements, and empirical convergence to the true source location under the adaptive sensor placement strategy. In all cases we evaluate the models on samples draw randomly from the same parameter distribution used to generate the training data.

\subsection{Inverse problem}
\label{sec:results_ip}
We demonstrate the ability of the CNWF model to produce accurate and interpretable source distribution predictions from the sparse sensor measurements by considering architected and intentionally ill-posed scenarios shown in Figures \ref{fig:circle_sweep1}, and \ref{fig:circle_sweep2}. In these scenarios we consider collinear sensor locations transversing the circular domain; when all sensors are located off the informative region within the advective plume, the SNR is lower and accurate prediction is infeasible. In both cases, the CNWF model produces smooth, interpretable, and qualitatively accurate source distributions, even in the ill-posed configurations. This also results in the lowest error as measured by the $W_2$ metric, in some cases by an order of magnitude over the baseline transformer. This confirms our first hypothesis~\ref{hyp:H1}, that the physical structure provides a useful inductive bias for training.

\begin{figure}[htbp]
    \centering
    \includegraphics[width=0.95\textwidth]{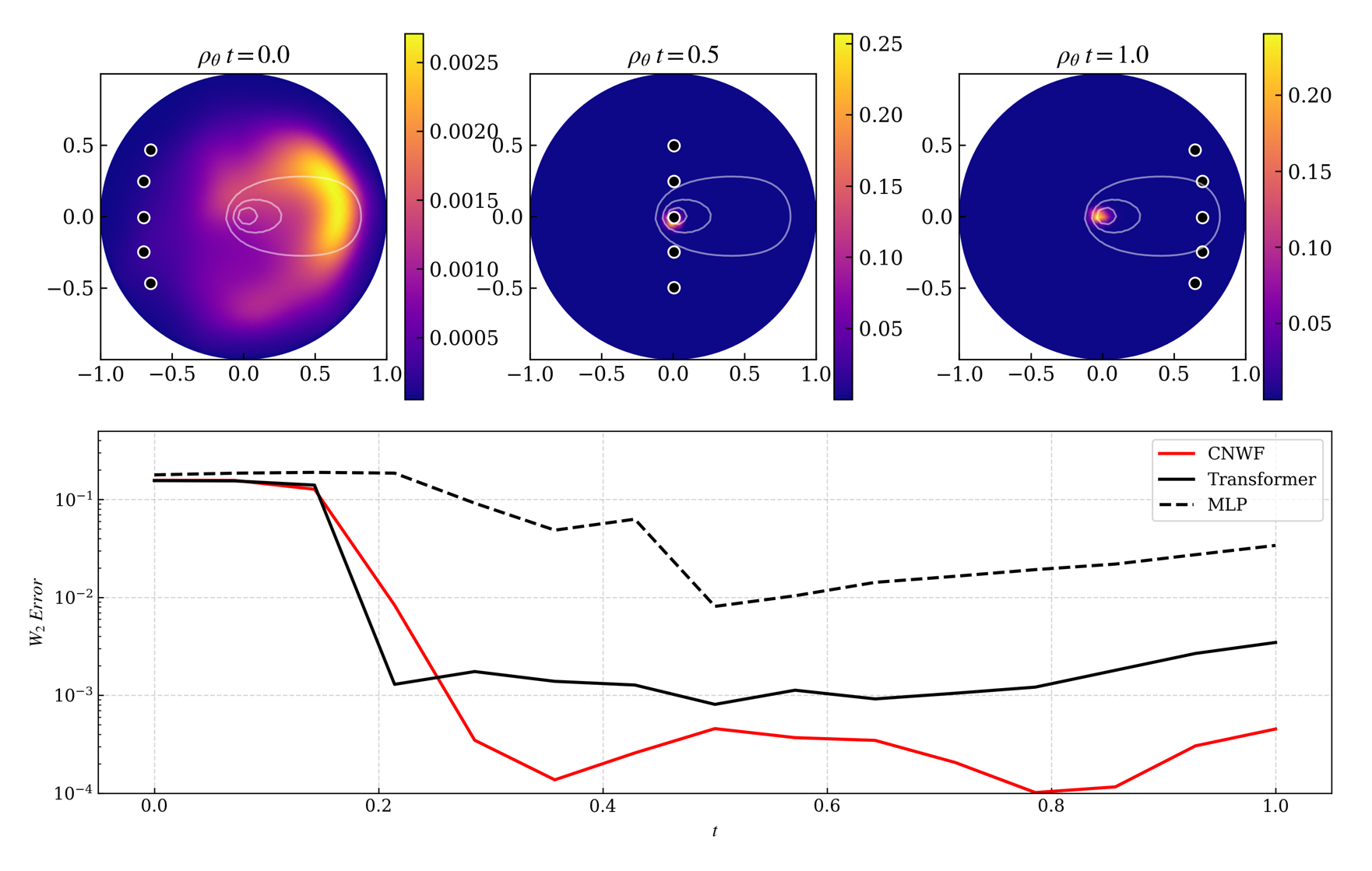}
    \caption{Sample model source predictions and predictions errors as sensors sweep across a horizontal plume. Five sensors are shown as black $o$'s, the scalar field is shown in a contour. We note that the initial configuration is ill-posed as none of the sensors located within the plume region.}
    \label{fig:circle_sweep1}
\end{figure}

\begin{figure}[htbp]
    \centering
    \includegraphics[width=0.95\textwidth]{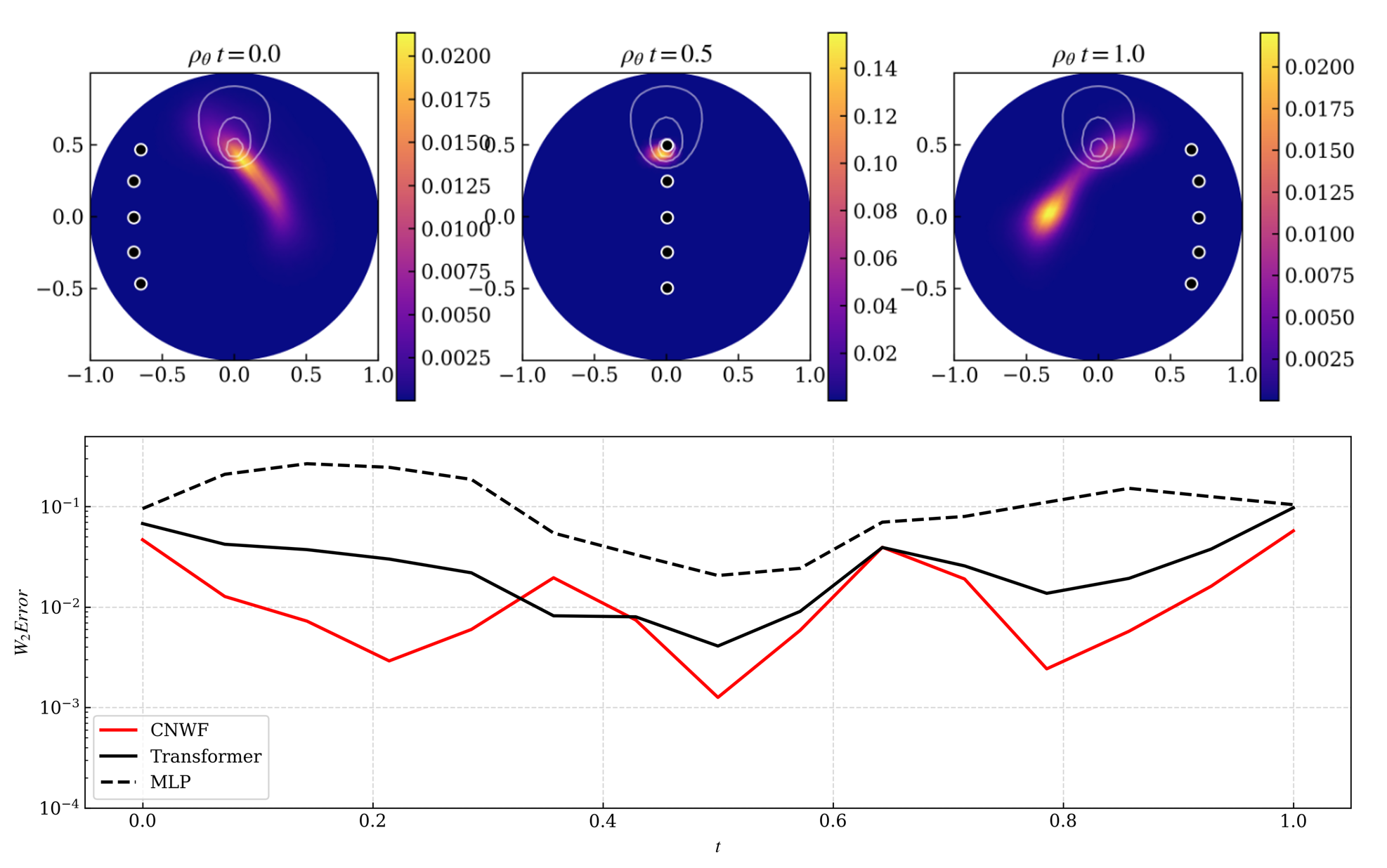}
    \caption{Sample model source predictions and predictions errors as sensors sweep across a vertical plume located near the top of the domain. Five sensors are shown as black $o$'s, the scalar field is shown in a contour. We note that although the initial and final configurations are ill-posed since none of the sensors located within the plume region, the CNWF model still produces an interpretable probability distributions over possible source locations.}
    \label{fig:circle_sweep2}
\end{figure}

In Figure \ref{fig:source_pred_hist} we compare the average predicted source density for each model, over 1024 random sensor configurations. We note that while the CNWF and transformer baseline both capture the area of highest probability, corresponding to the true source location, the CNWF localizes the source better on average, and provides a smoother source field.

\begin{figure}[htbp]
    \centering
    \includegraphics[width=0.65\textwidth]{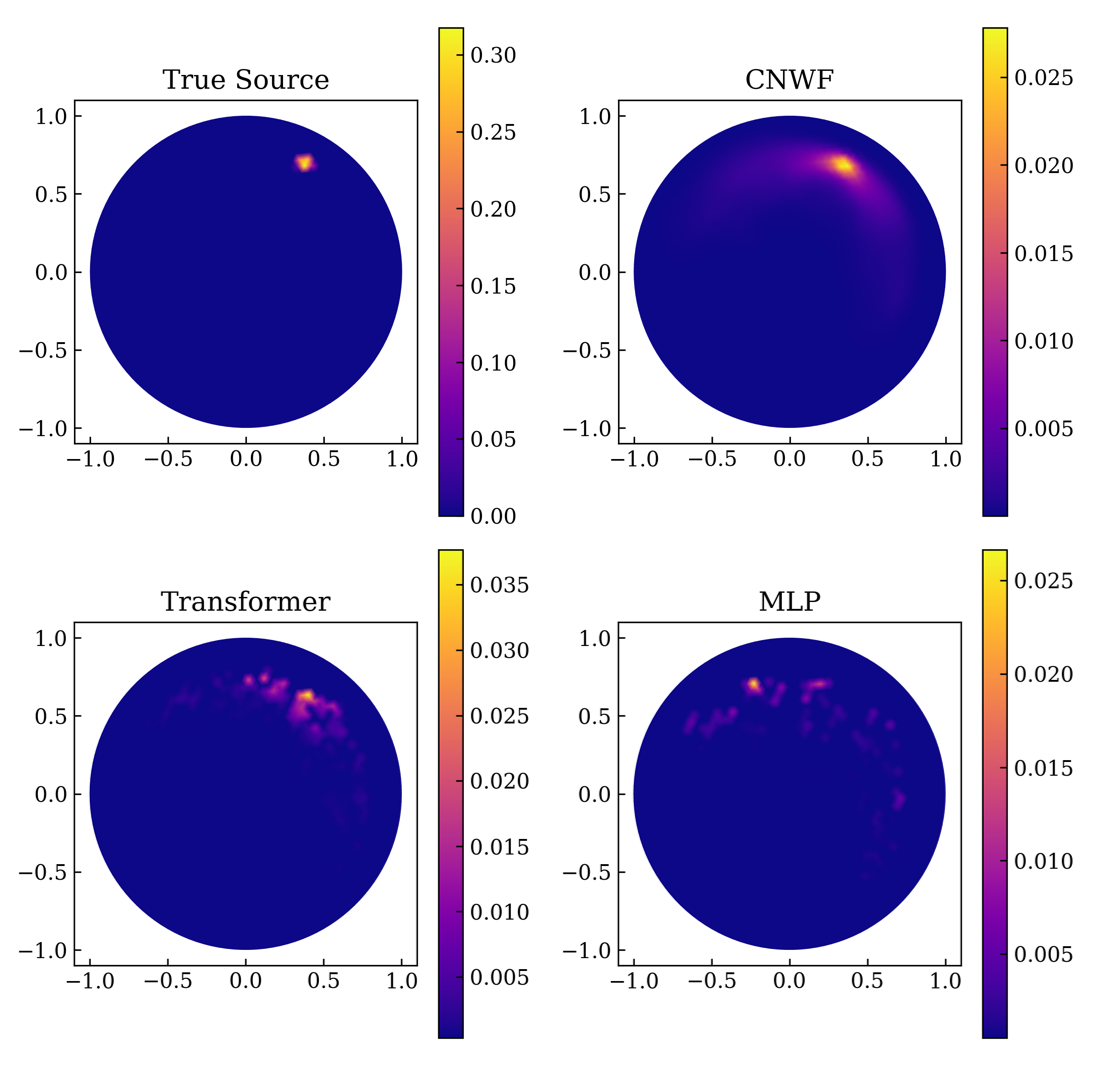}
    \caption{We consider the average source prediction over 2048 random velocity directions and sensor configurations for a fixed point source location. The true source term is shown in the top left panel. We note that the CNWF recovers a smooth and well behaved probability density, while the transformer and MLP baselines result in scattered predictions; this demonstrates the benefit of the adaptive basis for localization. }
    \label{fig:source_pred_hist}
\end{figure}

We hypothesis that the regularity of the CNWF generated predictions is the use of the adaptive basis functions, which implicitly bias the model towards better localization. 
We visualize the basis functions for a sample source and sensor configuration on the circular domain in Figure \ref{fig:gulf_POUS}. A shared basis is used to represent the source and solution during training. We note the correspondence of a single partition to the source location ($\psi_4^0$) while others correspond to the predicted downstream solution.

\begin{figure}[htbp]
    \centering
    \includegraphics[width=0.85\textwidth]{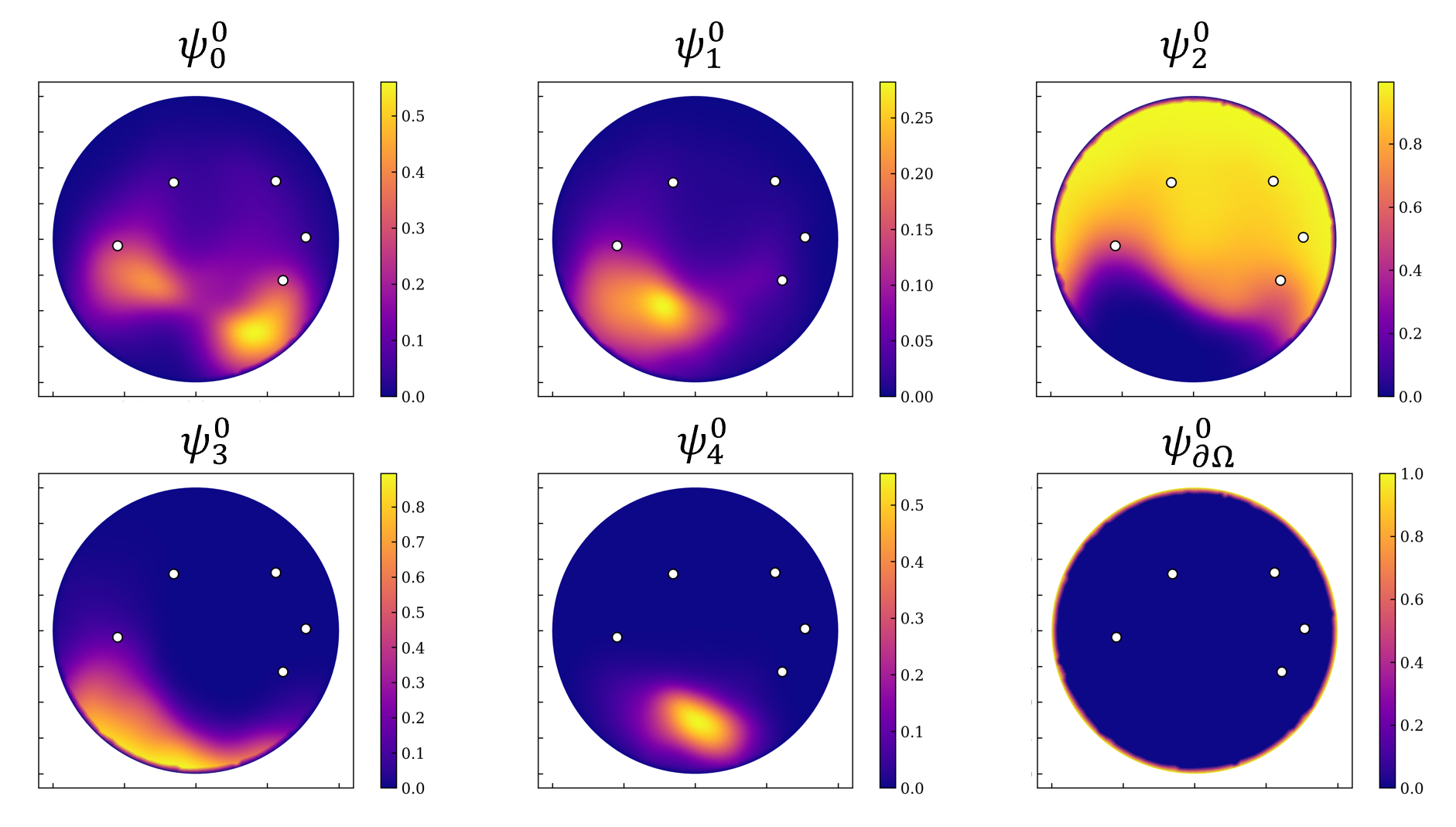}
    \caption{Learned basis functions for the CNWF model on the gulf domain. Each subplot corresponds to a learned partition-of-unity, $\varphi_i$, with the exception of $\psi_{\partial\Omega}^0$, which is a prescribed Dirichlet boundary partition. The sensor locations are indicated with white markers.}
    \label{fig:gulf_POUS}
\end{figure}

To assess physical consistency we evaluate the full inverse-forward process. The inverse model returns a normalized source density
\(
  \hat f_\theta(\mathbf{x}) = \rho_\theta(\mathbf{x}\,|\,z)
\);
we input this prediction into the original steady advection–diffusion equation \eqref{eq:advec-diff}
using the true velocity field \(\mathbf v\) and P\'eclet number.
This is then solved using the same FEM discretization on the fine mesh as for the data generation, producing the reconstructed scalar field \(\tilde u\).
Consistency is quantified by the normalized MSE between the reconstructed solution and true solution using a normalized input source,
\[
  E_{consistency}\;=\;
  \frac{\bigl\|\,\tilde u-u_{\mathrm{true}}\bigr\|_{L^{2}(\Omega)}}
       {\bigl\|u_{\mathrm{true}}\bigr\|_{L^{2}(\Omega)}}.
\]
Note that for the CNWF model, this is distinct from the forward FEM solve using the learned nonlinear flux.
We compare $E_{consistency}$ for the the CNWF and baseline models in Figure \ref{fig:u_error} and give sample solutions in the circular domain in Figure \ref{fig:u_tilde_comp}. We observe that in addition to yielding more accurate source predictions, the CNWF model provides a source term which is more compatible with the original PDE. This is most pronounced in the circular domain experiment.

\begin{figure}[htbp]
    \centering
    \includegraphics[width=0.95\textwidth]{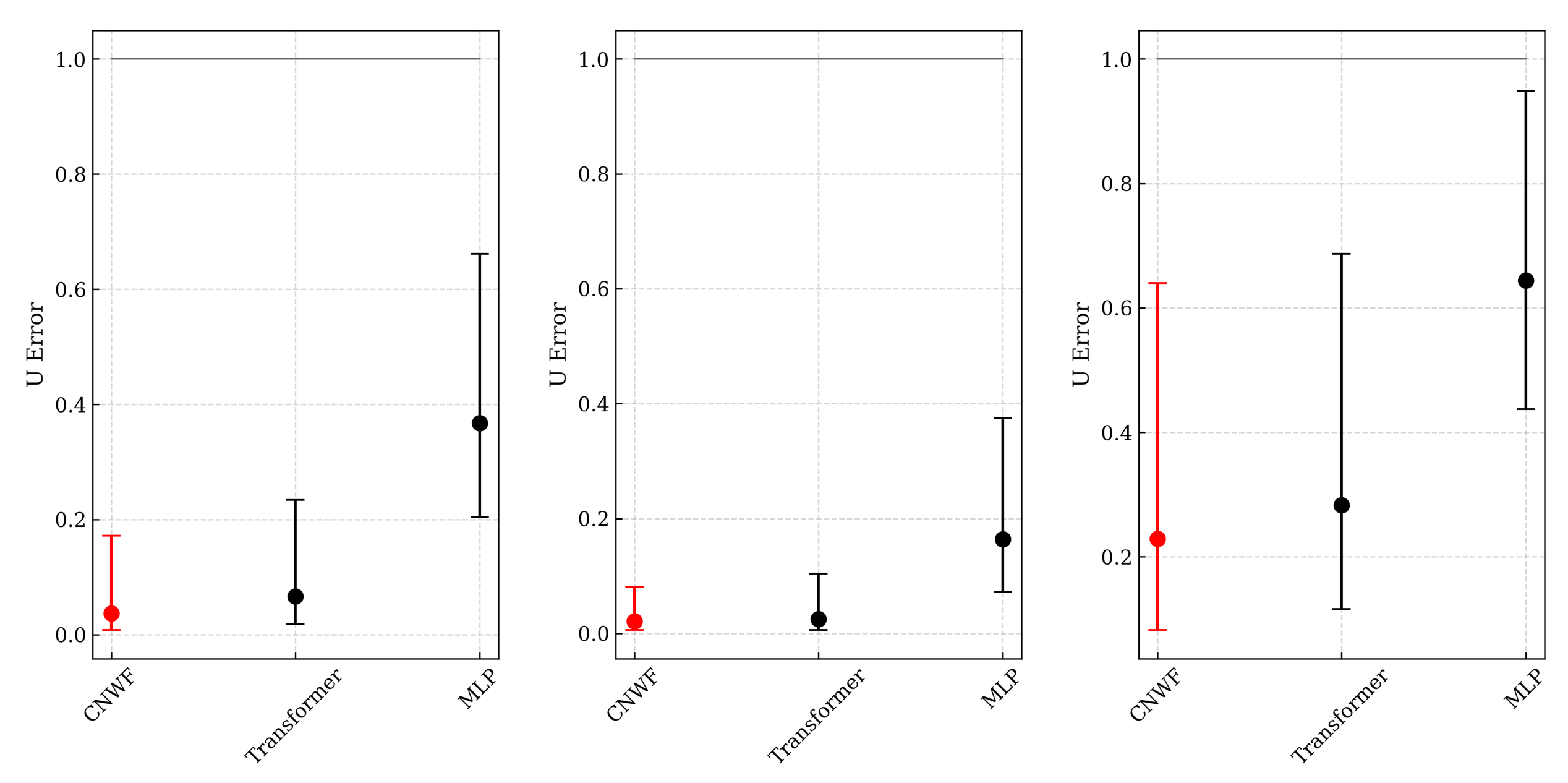}
    \caption{Evaluation of compatibility of the learned source term by forward evaluation in the original advection diffusion equation via FEM, for the CNWF and baseline models across each example case. }
    \label{fig:u_error}
\end{figure}

\begin{figure}[htbp]
    \centering
    \includegraphics[width=0.55\textwidth]{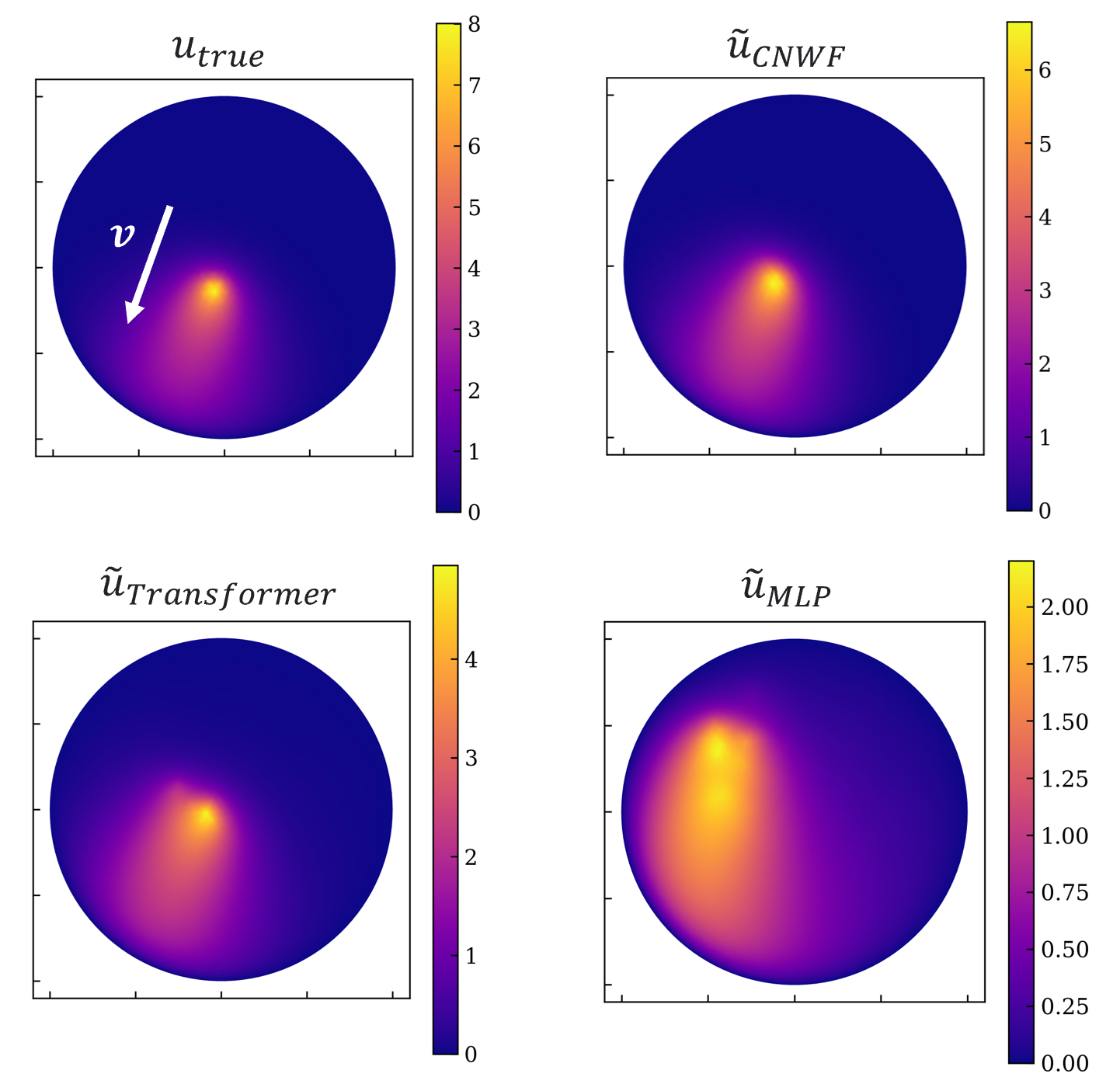}
    \caption{Comparison of forward solutions from predicted source terms for each model compared to the baseline solution. Note that the color scales are different across images.}
    \label{fig:u_tilde_comp}
\end{figure}

We vary the number of sensors \( N \) to evaluate how prediction accuracy scales with sensor count.
Despite being trained on a fixed number of sensors, the model exhibits strong generalization to varying sensor counts due to the permutation-invariant architecture of the transformer encoder. As shown in Fig~\ref{fig:sensor_convergence_comb}, the average prediction error for the inferred source \( f_h \) decreases monotonically with the number of available sensors. This reflects the model’s ability to effectively integrate information across an expanding set of observations, even without retraining. While the error does not converge to zero due to factors such as measurement noise, limited spatial coverage, and model representational capacity, the trend confirms that the architecture maintains stable performance under partial supervision and supports variable sensor deployment in practical applications.
We note the CNWF improves faster with additional sensors than the transformer baseline, and reaches a lower minimum.

\begin{figure}[htbp]
    \centering
    \includegraphics[width=0.6\textwidth]{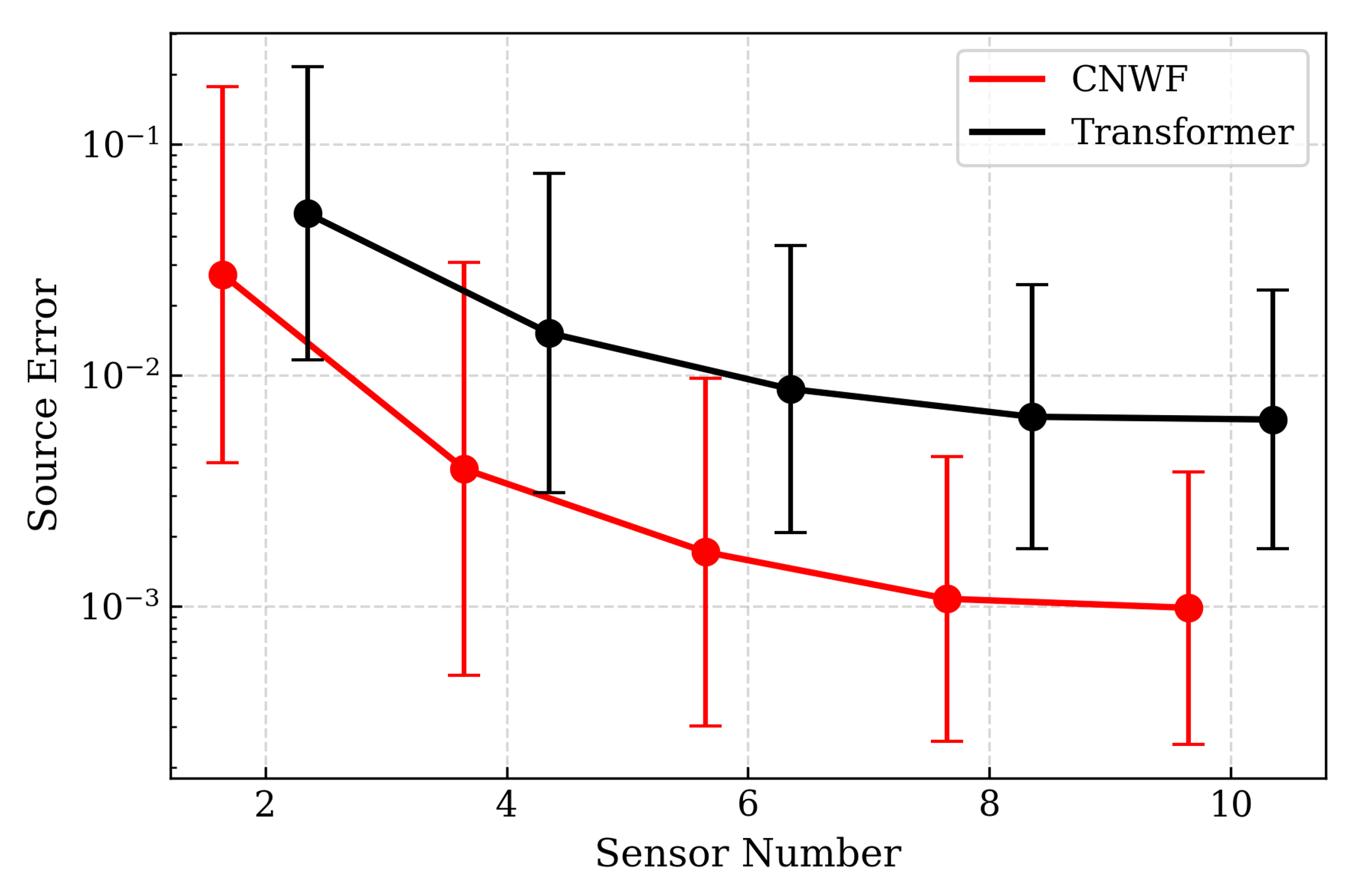}
    \caption{We compare the normalized source prediction MSE for the CNWF and transformer models over a variable static sensor number for the circular domain. Both models improve nearly monotonically with increased sensor number, but the CNWF demonstrates faster convergence, lower average and minimum error. }
    \label{fig:sensor_convergence_comb}
\end{figure}

In Theorem \ref{thm:3} we hypothesize that reducing the true coverage error results in more informative sensor configurations, and therefore improved error bounds on the source prediction task.
We additionally analyze model performance as a function of the average observed field magnitude at sensor locations:
\[
\bar{u}_\text{obs} = \frac{1}{N} \sum_{i=1}^N | u(\mathbf{x}_i) |.
\]
We interpret $\bar{u}_\text{obs}$ as an empirical proxy for the local observability of the source, with higher values indicating stronger coupling between sensors and the underlying field dynamics. In the presence of noise, this directly correlates to the signal-to-noise ratio (SNR), where higher sensor value indicates a lower SNR.

As shown in Fig.~\ref{fig:info_convergence}, decreased $\mathcal{J}_{true}$ and increased $u_{obs}$ both result in reduced source prediction error, suggesting that access to higher quality measurements improves inference. 
The observed dependence of accuracy on \( \mathcal{J}_{true} \) motivates the use of the adaptive coverage objective described in section \ref{sec:methodology_two_step}, under the condition in \ref{thm:3}.

\begin{figure}[htbp]
    \centering
    \includegraphics[width=0.75\textwidth]{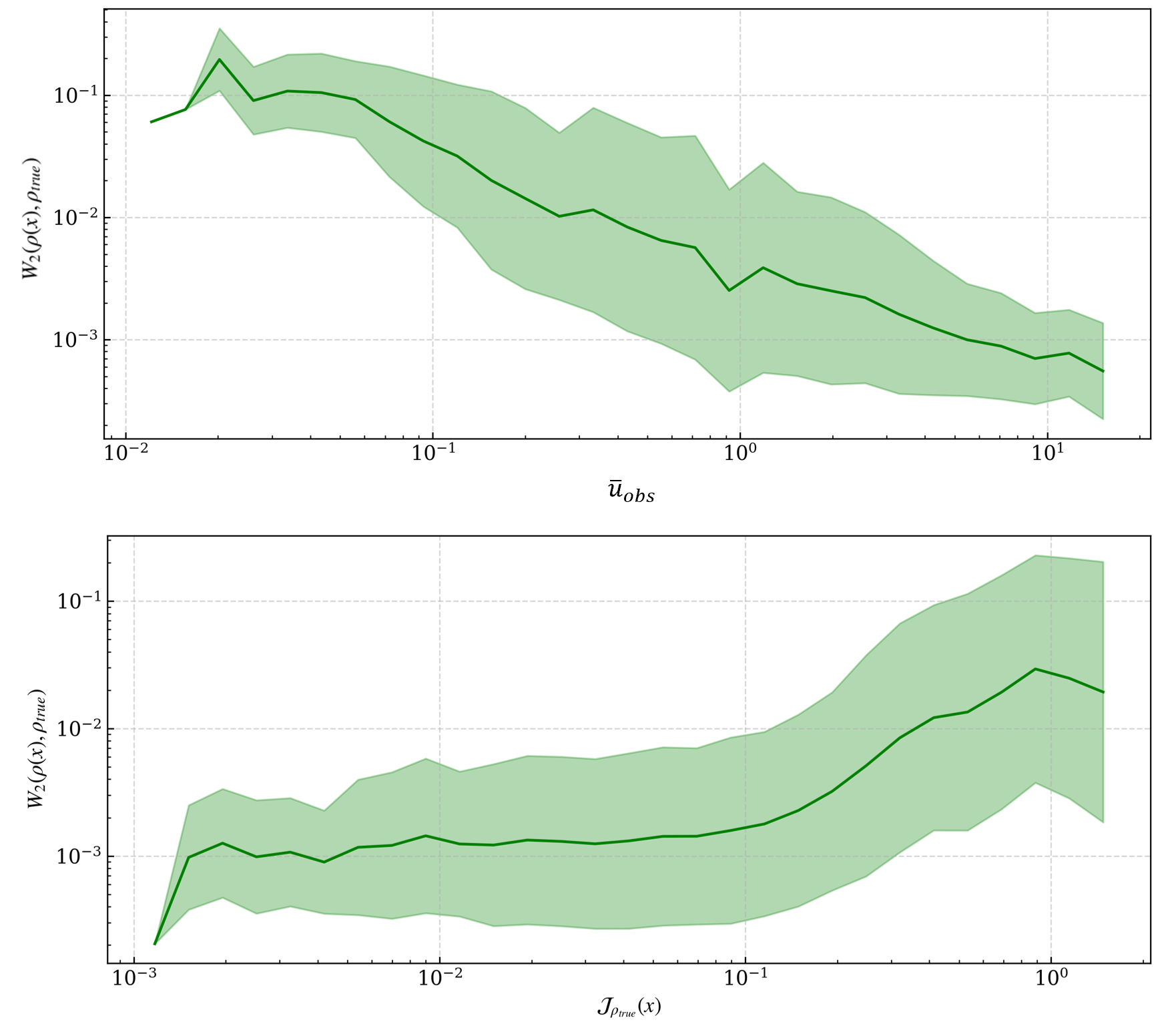}
    \caption{We demonstrate the CNWF model convergence for field prediction and source prediction over sensor measurement quality for all experiments. We use both the total observed concentration $\tilde{u}_{obs}$ (top) and the true coverage energy $\mathcal{J}_{true}$ (bottom) as measures of the sensor informativeness. Increasing the average sensor reading or decreasing the coverage energy both improve source prediction accuracy.}
    \label{fig:info_convergence}
\end{figure}

\subsection{Nonlinear Advection Diffusion}
To demonstrate the generality of the CNWF to a general class of transport processes we also consider a nonlinear advection diffusion system, with
\begin{equation}
    \mathbf{v}(u):=
    \begin{cases}
        0 & u(x)<T,
        \\ \mathbf{v}_0 \alpha u(x) & u(x) > T.
    \end{cases}
\end{equation}
and demonstrate similar, improved performance in this setting, as shown in table~\ref{tab:nonlinear_tabl}.

\begin{table}[t]
    \centering
    \caption{Test performance for nonlinear advection on circular geometry by Wasserstein metric.}
    \label{tab:nonlinear_tabl}
    \begin{tabular}{lc}
    \toprule
    \textbf{Model} & $W_2(\rho_\theta^0, \rho_{true})$ $\downarrow$ \\
    \midrule
    CNWF (ours) & \textbf{1.16e-3}\\
    Transformer baseline & 7.37e-3\\
    MLP baseline & 1.59e-2\\
    \bottomrule
    \end{tabular}
\end{table}

\subsection{Adaptive sensor placement}
\label{sec:results_sp}

We evaluate the ability of the proposed method to improve source localization by iteratively refining sensor positions using the geodesic Lloyd algorithm. Starting from a fixed initial configuration, sensors are repositioned over $K$ outer iterations, using the model-predicted source density $\rho_\theta(\cdot \mid z_k)$ as an importance function. We track both the model-relative objective $J_{\rho_k}(X)$ and the true localization error $J_{\mathrm{true}}(X)$ across iterations. Sample trajectories for a single randomly selected case on each geometry is shown in Figure \ref{fig:lloyd2}

\begin{figure}[htbp]
    \centering
    \includegraphics[width=0.85\textwidth]{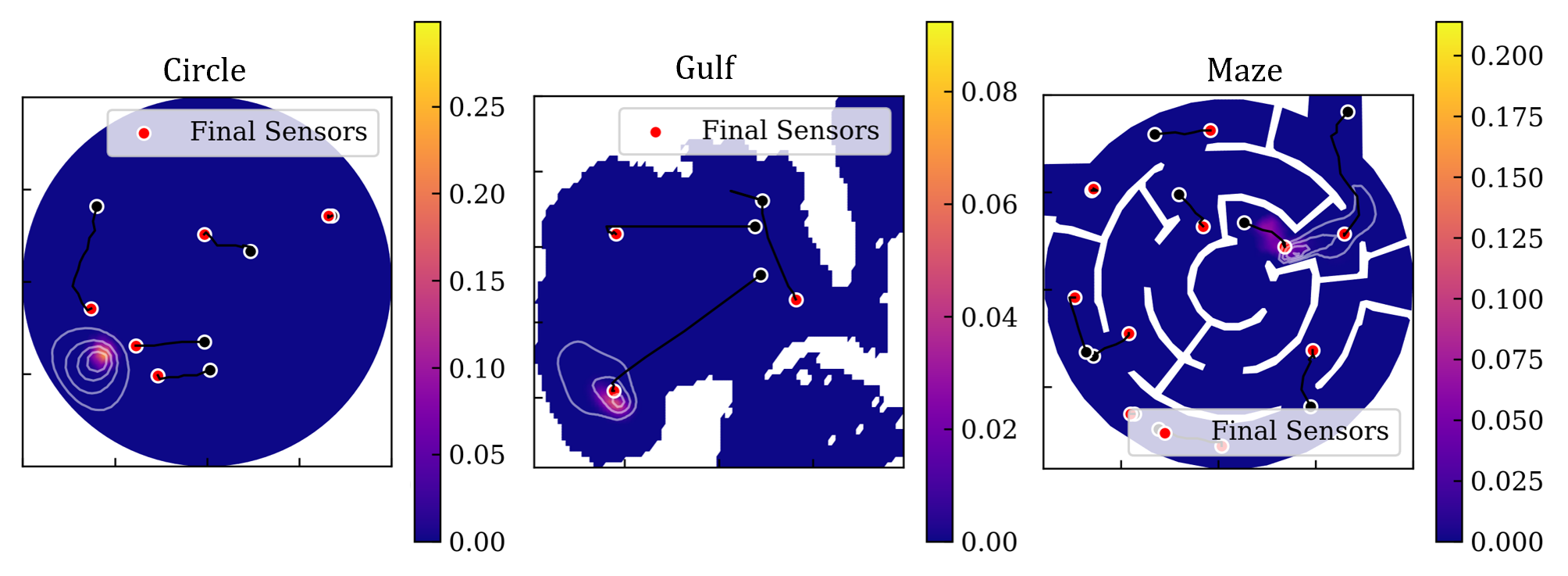}
    \caption{Example iterative CNWF-Lloyds (alg. \ref{alg:outer}) runs for each geometry. Initial sensor locations are shown in black, and final locations in red. The source is shown in color and the resulting scalar field is indicated with a white contour. These examples were randomly selected from cases with poor initial configurations.}
    \label{fig:lloyd2}
\end{figure}

In Figure \ref{fig:results_overview} we compare the average prediction errors (by Wasserstein metric) for random initializations and after running the two step adaptive sampling algorithm described in Algorithm \ref{alg:outer}. In all cases adaptive sampling improves the final source identification. The CNWF model outperforms the baselines both before and after sensor positioning, and demonstrates the largest proportional improvement from adaptive sampling.
We note that the changes due to sensor positioning are smaller than the differences between model class, indicating that while good sensor placement is clearly necessary, model capability is a larger limiting factor. This underscores the need for a principled training procedure such as in the CNWF implementation.

\begin{figure}[htbp]
    \centering
    \includegraphics[width=0.95\textwidth]{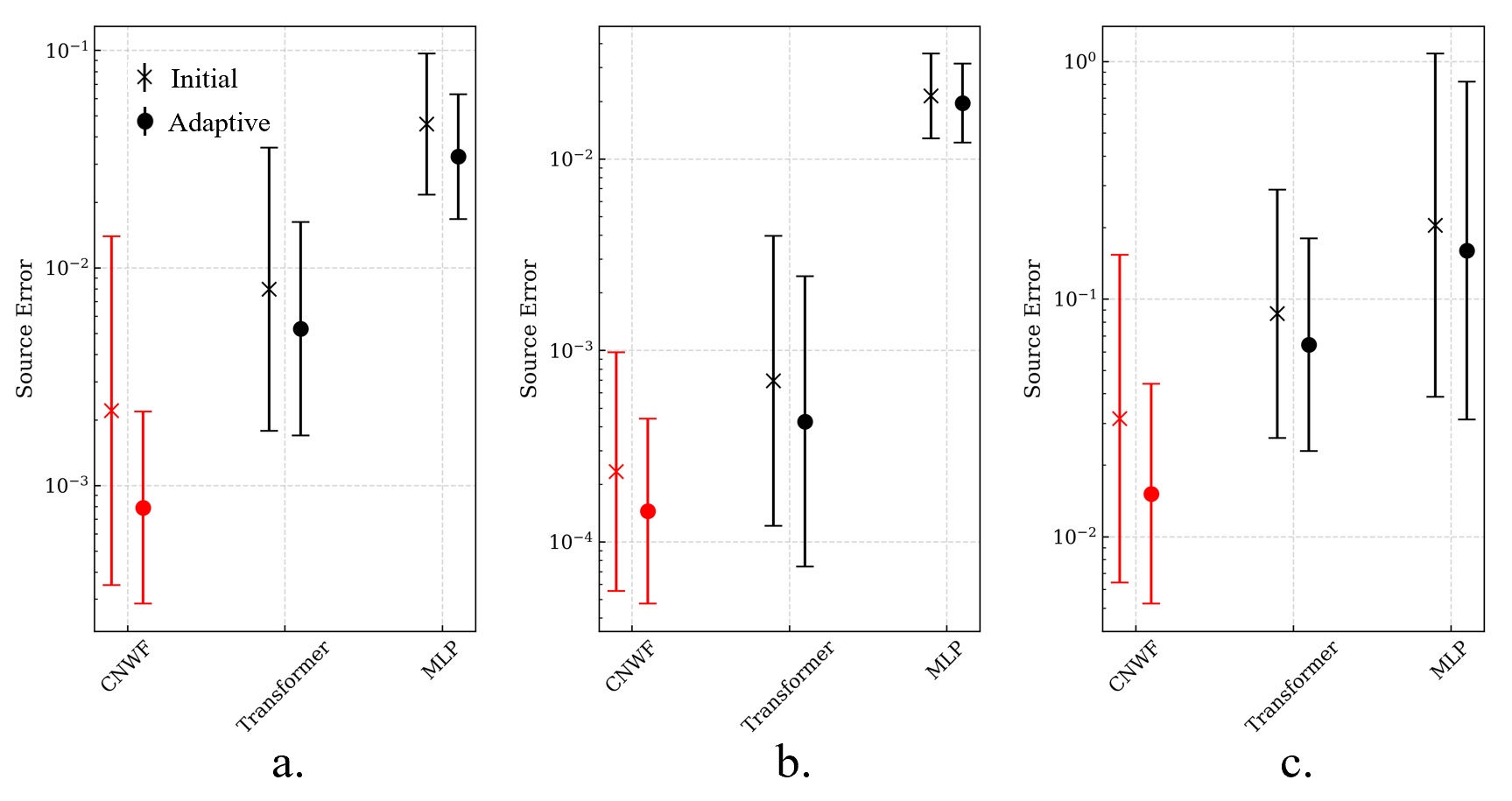}
    \caption{Comparison of the source prediction accuracy according to the Wasserstein metric, $W_{\epsilon,2}(\rho_\theta,\rho_{true})$, over the three experiments and three model architectures for each. We compare the prediction error before and after running the Lloyd's sensor placement algorithm, with the original values indicated with an ``$\text{x}$'' and the updated value using adaptive sampling indicated with an ``$\text{o}$''. The CNWF model outperforms the baselines both before and after sensor positioning, and demonstrates the largest proportional improvement from adaptive sampling. In all cases the use of the sensor positioning algorithm results in improved source prediction for all considered architectures, with the largest proportional improvement for the CNWF model.}
    \label{fig:results_overview}
\end{figure}

Figure~\ref{fig:lloyd_j} reports normalized values of both $\mathcal{J}_{\rho_k}$ and $\mathcal{J}_{\mathrm{true}}$ as a function of iteration $k$ for the circular domain over 20 trials. Model updates to $\rho_k$ are applied on iterations 5, 11, and 17.
We observe near-monotonic improvement in all test cases, suggesting that the CNWF-predicted source field remains stable and informative throughout the adaptive loop. The change in coverage objective is shown in Figure~\ref{fig:lloyd_delta_j}.

\begin{figure}[htbp]
    \centering
    \includegraphics[width=0.75\textwidth]{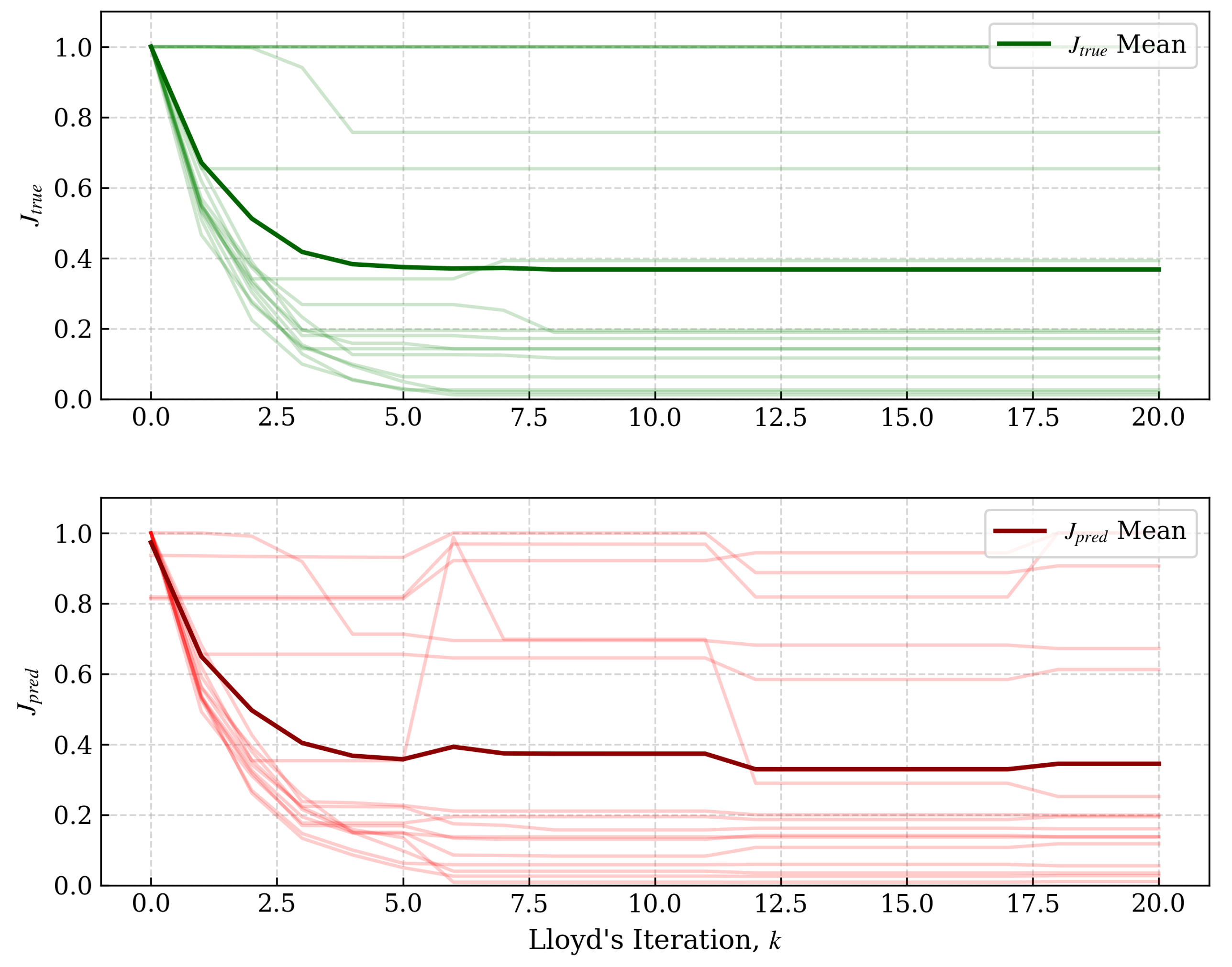}
    \caption{Normalized true and model relative coverage costs over Lloyd iterations for the circular domain, with importance updates at iteration $k\in[5,11,17]$. The average over the 20 trials is show in bold.}
    \label{fig:lloyd_j}
\end{figure}

\begin{figure}[htbp]
    \centering
    \includegraphics[width=0.75\textwidth]{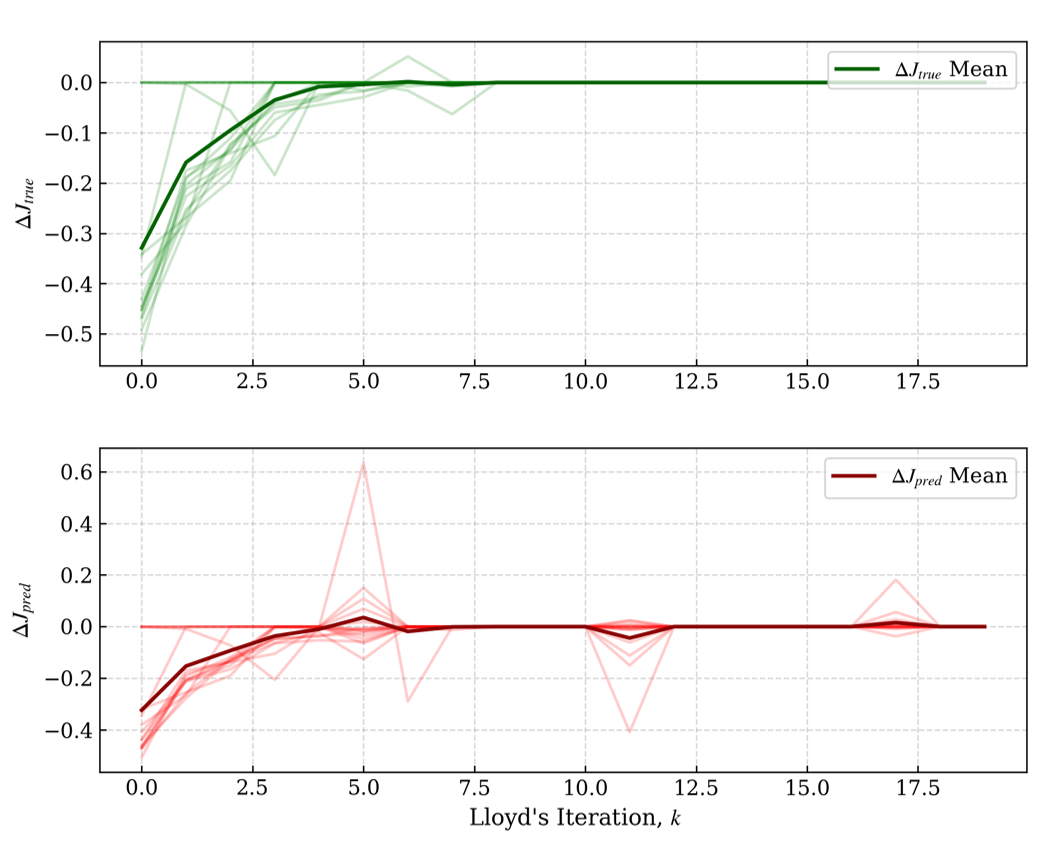}
    \caption{Change in coverage costs from Figure \ref{fig:lloyd_j}.}
    \label{fig:lloyd_delta_j}
\end{figure}

This closed-loop behavior confirms the benefit of integrating structure-preserving inference with sensor guidance: learned source predictions not only match the observed field but also serve as actionable importance maps for adaptive coverage. As shown in Table~\ref{tab:vertical_results}, this process improves final source prediction accuracy across all geometries. The improvement is most pronounced in the circular domain, where velocity variability is high and field observations alone are less informative without adaptive refinement.
The key metrics are summarized in table \ref{tab:vertical_results}.

\begin{table}[t]
    \centering
    \caption{Test performance on each geometry.  
    Metrics: KL divergence (KL), root-mean-square error on the source (F-RMSE), and RMSE of the scalar field obtained by solving the forward PDE with the predicted source (U-RMSE).  $\downarrow$ indicates lower is better.}
    \label{tab:vertical_results}
    \begin{tabular}{llccc}
    \toprule
    \textbf{Geometry} & \textbf{Model}
    & $W_2(\rho_\theta^0, \rho_{true})$ $\downarrow$ & $W_2(\rho_\theta^K, \rho_{true})$ $\downarrow$ & $\frac{\bigl\|\,\tilde u-u_{\mathrm{true}}\bigr\|_{L^{2}(\Omega)}}
           {\bigl\|u_{\mathrm{true}}\bigr\|_{L^{2}(\Omega)}}$ $\downarrow$ \\
    \midrule
    \multirow{3}{*}{Circle}%
        & CNWF (ours) & \textbf{2.20e-3} & \textbf{7.90e-4} & \textbf{3.77e-2} \\
        & Transformer baseline & 7.98e-3 & 5.25e-3 & 6.70e-2 \\
        & MLP baseline & 4.58e-2 & 3.25e-2 & 3.68e-1 \\[2pt]
        \midrule
    \multirow{3}{*}{Gulf}%
        & CNWF (ours) & \textbf{2.32e-4} & \textbf{1.44e-4} & \textbf{2.14e-2} \\
        & Transformer baseline & 6.93e-4 & 4.25e-4 & 2.45e-2 \\
        & MLP baseline & 2.13e-2 & 1.96e-2 & 1.64e-1 \\[2pt]
        \midrule
    \multirow{3}{*}{Maze}%
        & CNWF (ours) & \textbf{3.13e-2} & \textbf{1.51e-2} & \textbf{2.29e-1} \\
        & Transformer baseline & 8.66e-2 & 6.42e-2 & 2.83e-1 \\
        & MLP baseline & 2.04e-1 & 1.60e-1 & 6.44e-1 \\[2pt]
        \midrule
    \bottomrule
\end{tabular}
\end{table}

\section{Discussion and conclusions}
\label{sec:conclusions}

We presented a structure-preserving framework for source identification in advection–diffusion systems, coupling finite element exterior calculus with neural network surrogates to recover spatially distributed source terms from sparse sensor measurements. The proposed CNWF model integrates learned source and flux terms within a PDE-constrained optimization framework that maintains compatibility with the underlying conservation law. The data-adaptive reduced basis, provides efficient inference while retaining physical consistency.

Our analysis established conditions under which the adaptive sensor placement loop, based on geodesic Lloyd’s algorithm, converges in both model-relative and true coverage objectives. In particular, we proved that if the change in the predicted importance function is bounded relative to the descent achieved in the Lloyd update, then the overall sensing configuration strictly improves. This feedback loop was shown to reduce prediction error bounds under mild regularity assumptions, providing theoretical support for the closed-loop design.

Empirical results across three domains demonstrate that CNWF achieves improved source prediction and consistency with the true PDE compared to baseline MLP and transformer models. Notably, the CNWF model exhibits strong generalization to unseen sensor configurations, improved localization through adaptive sampling, and smooth, interpretable source fields. These benefits result from structure preservation, reduced basis representation, and physically motivated training constraints. Together, these results validate CNWF as a principled and effective tool for inverse modeling and adaptive sensing in transport systems.

\section{ACKNOWLEDGMENTS}
B. Shaffer's work is supported by the National Science Foundation Graduate Research Fellowship under Grant No. DGE-2236662. M. A. Hsieh's work is supported by ONR Award \#N000142512171. N. Trask's work is supported by SEA-CROGS (Scalable, Efficient and Accelerated Causal Reasoning Operators, Graphs and Spikes for Earth and Embedded Systems), a Mathematical Multifaceted Integrated Capability Center (MMICCs) funded by the Department of Energy Office of Science. J. Klobusicky's work is supported by the National
Science Foundation under Grant No.\ 231628. The authors thank Quercus Hernandez for assistance designing Figure 1.


\bibliographystyle{plain}
\bibliography{bib}

\begin{thebibliography}{10}

\bibitem{actor2024data}
Jonas~A Actor, Xiaozhe Hu, Andy Huang, Scott~A Roberts, and Nathaniel Trask.
\newblock Data-driven whitney forms for structure-preserving control volume analysis.
\newblock {\em Journal of Computational Physics}, 496:112520, 2024.

\bibitem{afsari2011riemannian}
Bijan Afsari.
\newblock Riemannian lp center of mass: existence, uniqueness, and convexity.
\newblock {\em Proceedings of the American Mathematical Society}, 139(2):655--673, 2011.

\bibitem{alpay2000model}
Mehmet~E Alpay and Molly~H Shor.
\newblock Model-based solution techniques for the source localization problem.
\newblock {\em IEEE transactions on control systems technology}, 8(6):895--904, 2000.

\bibitem{arnold2018finite}
Douglas~N Arnold.
\newblock {\em Finite element exterior calculus}.
\newblock SIAM, 2018.

\bibitem{arnold2006finite}
Douglas~N Arnold, Richard~S Falk, and Ragnar Winther.
\newblock Finite element exterior calculus, homological techniques, and applications.
\newblock {\em Acta numerica}, 15:1--155, 2006.

\bibitem{arridge2019solving}
Simon Arridge, Peter Maass, Ozan {\"O}ktem, and Carola-Bibiane Sch{\"o}nlieb.
\newblock Solving inverse problems using data-driven models.
\newblock {\em Acta Numerica}, 28:1--174, 2019.

\bibitem{bakushinsky2005iterative}
Anatolii~Borisovich Bakushinsky and M~Yu Kokurin.
\newblock {\em Iterative methods for approximate solution of inverse problems}, volume 577.
\newblock Springer Science \& Business Media, 2005.

\bibitem{bal2007inverse}
Guillaume Bal and Alexandru Tamasan.
\newblock Inverse source problems in transport equations.
\newblock {\em SIAM journal on mathematical analysis}, 39(1):57--76, 2007.

\bibitem{brandstetter2022clifford}
Johannes Brandstetter, Rianne van~den Berg, Max Welling, and Jayesh~K Gupta.
\newblock Clifford neural layers for pde modeling.
\newblock {\em arXiv preprint arXiv:2209.04934}, 2022.

\bibitem{cai2021physics}
Shengze Cai, Zhiping Mao, Zhicheng Wang, Minglang Yin, and George~Em Karniadakis.
\newblock Physics-informed neural networks (pinns) for fluid mechanics: A review.
\newblock {\em Acta Mechanica Sinica}, 37(12):1727--1738, 2021.

\bibitem{chavent2010nonlinear}
Guy Chavent.
\newblock {\em Nonlinear least squares for inverse problems: theoretical foundations and step-by-step guide for applications}.
\newblock Springer Science \& Business Media, 2010.

\bibitem{chen2021neural}
Yuhan Chen, Takashi Matsubara, and Takaharu Yaguchi.
\newblock Neural symplectic form: Learning hamiltonian equations on general coordinate systems.
\newblock {\em Advances in Neural Information Processing Systems}, 34:16659--16670, 2021.

\bibitem{cohen2016group}
Taco Cohen and Max Welling.
\newblock Group equivariant convolutional networks.
\newblock In {\em International conference on machine learning}, pages 2990--2999. PMLR, 2016.

\bibitem{cortes2004coverage}
Jorge Cortes, Sonia Martinez, Timur Karatas, and Francesco Bullo.
\newblock Coverage control for mobile sensing networks.
\newblock {\em IEEE Transactions on robotics and Automation}, 20(2):243--255, 2004.

\bibitem{cuturi2013sinkhorn}
Marco Cuturi.
\newblock Sinkhorn distances: Lightspeed computation of optimal transport.
\newblock {\em Advances in neural information processing systems}, 26, 2013.

\bibitem{de2022deep}
Maarten~V de~Hoop, Matti Lassas, and Christopher~A Wong.
\newblock Deep learning architectures for nonlinear operator functions and nonlinear inverse problems.
\newblock {\em Mathematical Statistics and Learning}, 4(1):1--86, 2022.

\bibitem{desbrun2005discrete}
Mathieu Desbrun, Anil~N Hirani, Melvin Leok, and Jerrold~E Marsden.
\newblock Discrete exterior calculus.
\newblock {\em arXiv preprint math/0508341}, 2005.

\bibitem{dhariwal2004bacterium}
Amit Dhariwal, Gaurav~S Sukhatme, and Aristides~AG Requicha.
\newblock Bacterium-inspired robots for environmental monitoring.
\newblock In {\em IEEE International Conference on Robotics and Automation, 2004. Proceedings. ICRA'04. 2004}, volume~2, pages 1436--1443. IEEE, 2004.

\bibitem{du1999centroidal}
Qiang Du, Vance Faber, and Max Gunzburger.
\newblock Centroidal voronoi tessellations: Applications and algorithms.
\newblock {\em SIAM review}, 41(4):637--676, 1999.

\bibitem{engl2015regularization}
Heinz~W Engl and Ronny Ramlau.
\newblock Regularization of inverse problems.
\newblock In {\em Encyclopedia of applied and computational mathematics}, pages 1233--1241. Springer, 2015.

\bibitem{geuzaine2009gmsh}
Christophe Geuzaine and Jean-Fran{\c{c}}ois Remacle.
\newblock Gmsh: A 3-d finite element mesh generator with built-in pre-and post-processing facilities.
\newblock {\em International journal for numerical methods in engineering}, 79(11):1309--1331, 2009.

\bibitem{greydanus2019hamiltonian}
Samuel Greydanus, Misko Dzamba, and Jason Yosinski.
\newblock Hamiltonian neural networks.
\newblock {\em Advances in neural information processing systems}, 32, 2019.

\bibitem{gustafsson2020scikit}
Tom Gustafsson and Geordie~Drummond Mcbain.
\newblock scikit-fem: A python package for finite element assembly.
\newblock {\em Journal of Open Source Software}, 5(52):2369, 2020.

\bibitem{hajieghrary2016multi}
Hadi Hajieghrary, M~Ani Hsieh, and Ira~B Schwartz.
\newblock Multi-agent search for source localization in a turbulent medium.
\newblock {\em Physics Letters A}, 380(20):1698--1705, 2016.

\bibitem{hajieghrary2017information}
Hadi Hajieghrary, Daniel Mox, and M~Ani Hsieh.
\newblock Information theoretic source seeking strategies for multiagent plume tracking in turbulent fields.
\newblock {\em Journal of Marine Science and Engineering}, 5(1):3, 2017.

\bibitem{hernandez2021structure}
Quercus Hern{\'a}ndez, Alberto Bad{\'\i}as, David Gonz{\'a}lez, Francisco Chinesta, and El{\'\i}as Cueto.
\newblock Structure-preserving neural networks.
\newblock {\em Journal of Computational Physics}, 426:109950, 2021.

\bibitem{hogan2019hycom}
Pat Hogan.
\newblock Hycom surface velocity fields for the gulf of mexico and the florida straits at 1km resolution for january 2014 and july 2014, 2019.

\bibitem{isakov1990inverse}
Victor Isakov.
\newblock {\em Inverse source problems}.
\newblock Number~34. American Mathematical Soc., 1990.

\bibitem{jagtap2022physics}
Ameya~D Jagtap, Zhiping Mao, Nikolaus Adams, and George~Em Karniadakis.
\newblock Physics-informed neural networks for inverse problems in supersonic flows.
\newblock {\em Journal of Computational Physics}, 466:111402, 2022.

\bibitem{jiang2024structure}
Shuai Jiang, Jonas Actor, Scott Roberts, and Nathaniel Trask.
\newblock A structure-preserving domain decomposition method for data-driven modeling.
\newblock {\em arXiv preprint arXiv:2406.05571}, 2024.

\bibitem{kamyab2022deep}
Shima Kamyab, Zohreh Azimifar, Rasool Sabzi, and Paul Fieguth.
\newblock Deep learning methods for inverse problems.
\newblock {\em PeerJ Computer Science}, 8:e951, 2022.

\bibitem{karniadakis2021physics}
George~Em Karniadakis, Ioannis~G Kevrekidis, Lu~Lu, Paris Perdikaris, Sifan Wang, and Liu Yang.
\newblock Physics-informed machine learning.
\newblock {\em Nature Reviews Physics}, 3(6):422--440, 2021.

\bibitem{kennedy2019generalized}
James Kennedy, Airlie Chapman, and Peter~M Dower.
\newblock Generalized coverage control for time-varying density functions.
\newblock In {\em 2019 18th European Control Conference (ECC)}, pages 71--76. IEEE, 2019.

\bibitem{khodayi2019model}
Reza Khodayi-mehr, Wilkins Aquino, and Michael~M Zavlanos.
\newblock Model-based active source identification in complex environments.
\newblock {\em IEEE Transactions on Robotics}, 35(3):633--652, 2019.

\bibitem{kimmel1998computing}
Ron Kimmel and James~A Sethian.
\newblock Computing geodesic paths on manifolds.
\newblock {\em Proceedings of the national academy of Sciences}, 95(15):8431--8435, 1998.

\bibitem{kinch2025structure}
Brooks Kinch, Benjamin Shaffer, Elizabeth Armstrong, Michael Meehan, John Hewson, and Nathaniel Trask.
\newblock Structure-preserving digital twins via conditional neural whitney forms.
\newblock {\em arXiv preprint arXiv:2508.06981}, 2025.

\bibitem{kingma2014adam}
Diederik~P Kingma and Jimmy Ba.
\newblock Adam: A method for stochastic optimization.
\newblock {\em arXiv preprint arXiv:1412.6980}, 2014.

\bibitem{kontos2022machine}
Yiannis~N Kontos, Theodosios Kassandros, Konstantinos Perifanos, Marios Karampasis, Konstantinos~L Katsifarakis, and Kostas Karatzas.
\newblock Machine learning for groundwater pollution source identification and monitoring network optimization.
\newblock {\em Neural Computing and Applications}, 34(22):19515--19545, 2022.

\bibitem{lee2015multirobot}
Sung~G Lee, Yancy Diaz-Mercado, and Magnus Egerstedt.
\newblock Multirobot control using time-varying density functions.
\newblock {\em IEEE Transactions on robotics}, 31(2):489--493, 2015.

\bibitem{lee2013controlled}
Sung~G Lee and Magnus Egerstedt.
\newblock Controlled coverage using time-varying density functions.
\newblock {\em IFAC Proceedings Volumes}, 46(27):220--226, 2013.

\bibitem{li2020fourier}
Zongyi Li, Nikola Kovachki, Kamyar Azizzadenesheli, Burigede Liu, Kaushik Bhattacharya, Andrew Stuart, and Anima Anandkumar.
\newblock Fourier neural operator for parametric partial differential equations.
\newblock {\em arXiv preprint arXiv:2010.08895}, 2020.

\bibitem{linardatos2020explainable}
Pantelis Linardatos, Vasilis Papastefanopoulos, and Sotiris Kotsiantis.
\newblock Explainable ai: A review of machine learning interpretability methods.
\newblock {\em Entropy}, 23(1):18, 2020.

\bibitem{lloyd1982least}
Stuart Lloyd.
\newblock Least squares quantization in pcm.
\newblock {\em IEEE transactions on information theory}, 28(2):129--137, 1982.

\bibitem{lohi2021whitney}
Jonni Lohi and Lauri Kettunen.
\newblock Whitney forms and their extensions.
\newblock {\em Journal of Computational and Applied Mathematics}, 393:113520, 2021.

\bibitem{lu2021learning}
Lu~Lu, Pengzhan Jin, Guofei Pang, Zhongqiang Zhang, and George~Em Karniadakis.
\newblock Learning nonlinear operators via deeponet based on the universal approximation theorem of operators.
\newblock {\em Nature machine intelligence}, 3(3):218--229, 2021.

\bibitem{luo2019distributed}
Wenhao Luo, Changjoo Nam, George Kantor, and Katia Sycara.
\newblock Distributed environmental modeling and adaptive sampling for multi-robot sensor coverage.
\newblock In {\em Proceedings of the 18th International Conference on Autonomous Agents and MultiAgent Systems}, pages 1488--1496, 2019.

\bibitem{marques2002olfaction}
Lino Marques, Urbano Nunes, and An{\'i}bal~T de~Almeida.
\newblock Olfaction-based mobile robot navigation.
\newblock {\em Thin solid films}, 418(1):51--58, 2002.

\bibitem{matthes2005source}
J{\"o}rg Matthes, Lutz Groll, and Hubert~B Keller.
\newblock Source localization by spatially distributed electronic noses for advection and diffusion.
\newblock {\em IEEE Transactions on Signal Processing}, 53(5):1711--1719, 2005.

\bibitem{mishra2022estimates}
Siddhartha Mishra and Roberto Molinaro.
\newblock Estimates on the generalization error of physics-informed neural networks for approximating a class of inverse problems for pdes.
\newblock {\em IMA Journal of Numerical Analysis}, 42(2):981--1022, 2022.

\bibitem{mitchell1987discrete}
Joseph~SB Mitchell, David~M Mount, and Christos~H Papadimitriou.
\newblock The discrete geodesic problem.
\newblock {\em SIAM Journal on Computing}, 16(4):647--668, 1987.

\bibitem{molinaro2023neural}
Roberto Molinaro, Yunan Yang, Bj{\"o}rn Engquist, and Siddhartha Mishra.
\newblock Neural inverse operators for solving pde inverse problems.
\newblock {\em arXiv preprint arXiv:2301.11167}, 2023.

\bibitem{DBLP:journals/corr/abs-1912-01703}
Adam Paszke, Sam Gross, Francisco Massa, Adam Lerer, James Bradbury, Gregory Chanan, Trevor Killeen, Zeming Lin, Natalia Gimelshein, Luca Antiga, Alban Desmaison, Andreas K{\"{o}}pf, Edward~Z. Yang, Zach DeVito, Martin Raison, Alykhan Tejani, Sasank Chilamkurthy, Benoit Steiner, Lu~Fang, Junjie Bai, and Soumith Chintala.
\newblock Pytorch: An imperative style, high-performance deep learning library.
\newblock {\em CoRR}, abs/1912.01703, 2019.

\bibitem{patel2022thermodynamically}
Ravi~G Patel, Indu Manickam, Nathaniel~A Trask, Mitchell~A Wood, Myoungkyu Lee, Ignacio Tomas, and Eric~C Cyr.
\newblock Thermodynamically consistent physics-informed neural networks for hyperbolic systems.
\newblock {\em Journal of Computational Physics}, 449:110754, 2022.

\bibitem{raissi2019physics}
Maziar Raissi, Paris Perdikaris, and George~E Karniadakis.
\newblock Physics-informed neural networks: A deep learning framework for solving forward and inverse problems involving nonlinear partial differential equations.
\newblock {\em Journal of Computational physics}, 378:686--707, 2019.

\bibitem{reiter2017machine}
Austin Reiter and Muyinatu A~Lediju Bell.
\newblock A machine learning approach to identifying point source locations in photoacoustic data.
\newblock In {\em Photons Plus Ultrasound: Imaging and Sensing 2017}, volume 10064, pages 504--509. SPIE, 2017.

\bibitem{Renaud_nanomesh}
Nicolas Renaud, Stef Smeets, and Lars~J. Corbijn~van Willenswaard.
\newblock {nanomesh}.

\bibitem{rowan2025definition}
Conor Rowan and Alireza Doostan.
\newblock On the definition and importance of interpretability in scientific machine learning.
\newblock {\em arXiv preprint arXiv:2505.13510}, 2025.

\bibitem{russell2003comparison}
R~Andrew Russell, Alireza Bab-Hadiashar, Rod~L Shepherd, and Gordon~G Wallace.
\newblock A comparison of reactive robot chemotaxis algorithms.
\newblock {\em Robotics and Autonomous Systems}, 45(2):83--97, 2003.

\bibitem{santos2019decentralized}
Maria Santos, Siddharth Mayya, Gennaro Notomista, and Magnus Egerstedt.
\newblock Decentralized minimum-energy coverage control for time-varying density functions.
\newblock In {\em 2019 international symposium on multi-robot and multi-agent systems (MRS)}, pages 155--161. IEEE, 2019.

\bibitem{soares2015distributed}
Jorge~M Soares, A~Pedro Aguiar, Ant{\'o}nio~M Pascoal, and Alcherio Martinoli.
\newblock A distributed formation-based odor source localization algorithm-design, implementation, and wind tunnel evaluation.
\newblock In {\em 2015 IEEE International Conference on Robotics and Automation (ICRA)}, pages 1830--1836. IEEE, 2015.

\bibitem{stuart2010inverse}
Andrew~M Stuart.
\newblock Inverse problems: a bayesian perspective.
\newblock {\em Acta numerica}, 19:451--559, 2010.

\bibitem{tarantola2005inverse}
Albert Tarantola.
\newblock {\em Inverse problem theory and methods for model parameter estimation}.
\newblock SIAM, 2005.

\bibitem{trask2022enforcing}
Nathaniel Trask, Andy Huang, and Xiaozhe Hu.
\newblock Enforcing exact physics in scientific machine learning: a data-driven exterior calculus on graphs.
\newblock {\em Journal of Computational Physics}, 456:110969, 2022.

\bibitem{vesselinov2018contaminant}
Velimir~V Vesselinov, Boian~S Alexandrov, and Daniel O’Malley.
\newblock Contaminant source identification using semi-supervised machine learning.
\newblock {\em Journal of contaminant hydrology}, 212:134--142, 2018.

\bibitem{wang2023physics}
Shupeng Wang, Hui Zhang, and Xiaoyun Jiang.
\newblock Physics-informed neural network algorithm for solving forward and inverse problems of variable-order space-fractional advection--diffusion equations.
\newblock {\em Neurocomputing}, 535:64--82, 2023.

\bibitem{wang2021understanding}
Sifan Wang, Yujun Teng, and Paris Perdikaris.
\newblock Understanding and mitigating gradient flow pathologies in physics-informed neural networks.
\newblock {\em SIAM Journal on Scientific Computing}, 43(5):A3055--A3081, 2021.

\bibitem{webster2012bioinspired}
DR~Webster, KY~Volyanskyy, and MJ~Weissburg.
\newblock Bioinspired algorithm for autonomous sensor-driven guidance in turbulent chemical plumes.
\newblock {\em Bioinspiration \& biomimetics}, 7(3):036023, 2012.

\bibitem{weimer2009multiple}
James Weimer, Bruno Sinopoli, and Bruce Krogh.
\newblock Multiple source detection and localization in advection-diffusion processes using wireless sensor networks.
\newblock In {\em 2009 30th IEEE Real-Time Systems Symposium}, pages 333--342. IEEE, 2009.

\bibitem{willcox2023foundational}
Karen Willcox et~al.
\newblock Foundational research gaps and future directions for digital twins.
\newblock Technical report, National Academies of Sciences, Engineering, and Medicine, 2023.
\newblock DOI: \url{https://doi.org/10.17226/26894}.

\bibitem{yu2022gradient}
Jeremy Yu, Lu~Lu, Xuhui Meng, and George~Em Karniadakis.
\newblock Gradient-enhanced physics-informed neural networks for forward and inverse pde problems.
\newblock {\em Computer Methods in Applied Mechanics and Engineering}, 393:114823, 2022.

\bibitem{zarzhitsky2005distributed}
Dimitri Zarzhitsky, Diana~F Spears, and William~M Spears.
\newblock Distributed robotics approach to chemical plume tracing.
\newblock In {\em 2005 IEEE/RSJ International Conference on Intelligent Robots and Systems}, pages 4034--4039. IEEE, 2005.

\end{thebibliography}

\section*{Appendix A: Neural Network Parameterization Details}
\label{app:architecture}

We parameterize the source, flux, and basis terms in eq. \ref{eq:discrete_conservation} using neural networks conditioned on sparse sensor measurements.
The CNWF approach provides a structure-preserving framework that admits flexible neural parameterizations. The key properties that must be preserved are the POU property of the basis functions, and the non-negativity of the learned source. The anti-symmetry of the fluxes is automatically achieved through the coboundary operator.

We implement the parameterizations as task specific head networks around a universal transformer encoder, which contains $\sim$90\% of the parameters of the composite CNWF model. This maps from the conditioning inputs to a latent representation and allows for a shared representation of the underlying physical problem setting across the partition, flux, and source networks.. Let \( z \in \mathbb{R}^{d_z} \) denote the vector of sensor inputs, which is encoded by the permutation-invariant transformer encoder
\(
    \mathcal{E}_\theta : \mathbb{R}^{d_z} \to \mathbb{R}^{d_{\hat{z}}},
\)
yielding a latent representation \( \hat{z} = \mathcal{E}_\theta(z) \in \mathbb{R}^{d_{\hat{z}}} \).

The reduced-basis source term is predicted by a network
\(
    \mathcal{S}_\theta : \mathbb{R}^{d_{\hat{z}}} \to \mathbb{R}^{n_c^0},
\)
so that the learned source coefficients are given by \( \hat{f}_\theta(z) := \mathcal{S}_\theta(\hat{z}) \geq 0\). We enforce non-negative outputs through a ReLU activation on the final layer.
The flux correction in the reduced 1-form space is predicted by a network
\(
    \mathcal{N}_\theta : \mathbb{R}^{n_c^0} \times \mathbb{R}^{d_{\hat{z}}} \to \mathbb{R}^{n_c^1},
\)
with output defined as \( \mathcal{N}_\theta(\hat{u}, z) := \mathcal{N}_\theta(\hat{u}, \hat{z}) \), where \( \hat{u} \) are the reduced scalar field coefficients.
Finally, the basis transformation is defined by a network
\(
    \mathcal{P}_\theta : \mathbb{R}^{d_{\hat{z}}} \to \mathbb{R}^{n_c^0 \times n_f^0},
\)
producing a convex combination matrix via row-wise softmax:
\(
    W_\theta(z) := \text{softmax}_{\text{row}} \left( \mathcal{P}_\theta(\hat{z}) \right).
\)

In this work, we use a self-attention transformer to encode the sensor data, an MLP for the source term, a modified residual network for the fluxes, and cross attention with the spatial node coordinates for the partition network, for more detail see Appendix \ref{app:model_params}. The flux networks are implemented in a residual architecture comprising a linear skip connection and a nonlinear MLP,
\[
    \mathcal{N}_\theta(\hat{u}, z) = \mathcal{L}(\hat{u}, z) + \alpha\mathcal{N}_{nonlinear}(\hat{u}, z),
\]
with a tunable gain parameter \( \alpha > 0\) that controls the contribution of the nonlinear component and balances expressivity with numerical stability. A schematic representation of the neural network parameterization is shown in Figure \ref{fig:method_overview}.

We note that this approach is amendable to supervised pretraining via the transformer backbone, which can then be fine tuned through the head networks into a structure preserving model. This approach greatly reduces the number of training iterations that require a full nonlinear solve but may introduce undesirable local minima when transitioning to the structure preserving phase. We did not pretrain in this work.

All networks are implemented in pyTorch \cite{DBLP:journals/corr/abs-1912-01703} and optimized using Adam \cite{kingma2014adam}. All terms are differentiable and trainable via auto-differentiation. The CNWF model for each test case was trained on a Nvidia H200, for $\sim$3 hours wall time per example.

\subsubsection*{Transformer Encoder}

A full description of the partition transformer architecture is provided in \cite{kinch2025structure}, and closely follows a standard cross-attention encoder. For completeness we describe our sensor encoding transformer backbone implementation here, which is built on a standard multihead self attention architecture.
Let \( z = \{ z_i \}_{i=1}^N \) be a set of sensor observations, where each \( z_i = (x_i, u_i, \mathbf{v}_i) \in \mathbb{R}^d \) encodes spatial location, scalar field value, and local velocity. We project each \( z_i \) into a token \( h_i^{(0)} \in \mathbb{R}^{d_{\text{token}}} \) via a learned linear embedding layer.

The encoder \( \mathcal{E}_\phi \) is composed of \( L \) stacked transformer blocks. Each block applies multi-head self-attention followed by an MLP with residual connections:
\[
\begin{aligned}
    H^{(\ell+1)} &= \text{LayerNorm} \left( H^{(\ell)} + \text{MultiHeadSelfAttention}(H^{(\ell)}) \right), \\
    H^{(\ell+1)} &= \text{LayerNorm} \left( H^{(\ell+1)} + \text{MLP}(H^{(\ell+1)}) \right),
\end{aligned}
\]
where \( H^{(\ell)} = \{ h_i^{(\ell)} \}_{i=1}^N \), and attention is computed using:
\[
    \text{Attention}(Q, K, V) = \text{softmax} \left( \frac{QK^\top}{\sqrt{d_{\text{head}}}} \right) V.
\]

After the final attention layer, we obtain updated tokens \( \{ h_i^{(L)} \} \). These are pooled into a global latent \( \hat{z} \in \mathbb{R}^{d_z} \) using attention pooling:
\[
    \hat{z} = \sum_{i=1}^N \alpha_i h_i^{(L)}, \quad \alpha_i = \frac{\exp(w^\top h_i^{(L)})}{\sum_j \exp(w^\top h_j^{(L)})},
\]
with \( w \in \mathbb{R}^{d_{\text{token}}} \) learned during training.

This encoder is permutation-invariant by design and supports variable-size input sets.

\section*{Appendix B: Data setup}
All data were generated by solving the FEM discretization of the advection diffusion equations using the scikit-fem \cite{gustafsson2020scikit} python package. Meshes were generated using gmsh \cite{geuzaine2009gmsh}, or nanomesh \cite{Renaud_nanomesh}. Each mesh has $~O(1e3)$ nodes, although we note the reduced basis allows efficient generalization to finer discretizations. Velocitiy profiles for the Gulf were retrieved from NOAA \cite{hogan2019hycom}, and for the maze were generated by solving stokes flow in scikit-fem with inlet and outlet boundary conditions. We use homogeneous Dirichlet boundaries throughout, with Neumann conditions imposed on the outlet of the maze and gulf. We parameterize the source term with a bump function with radius $r=0.07$. We chose P\'eclet $~1e3$ for all examples to give qualitatively interesting results while maintaining stability for the chosen discretization and solver. The discretized domains, and sample velocities, forcings, and resulting scalar fields are shown for each experiment in Figure \ref{fig:data_overview}.

\begin{figure}[htbp]
    \centering
    \includegraphics[width=0.95\textwidth]{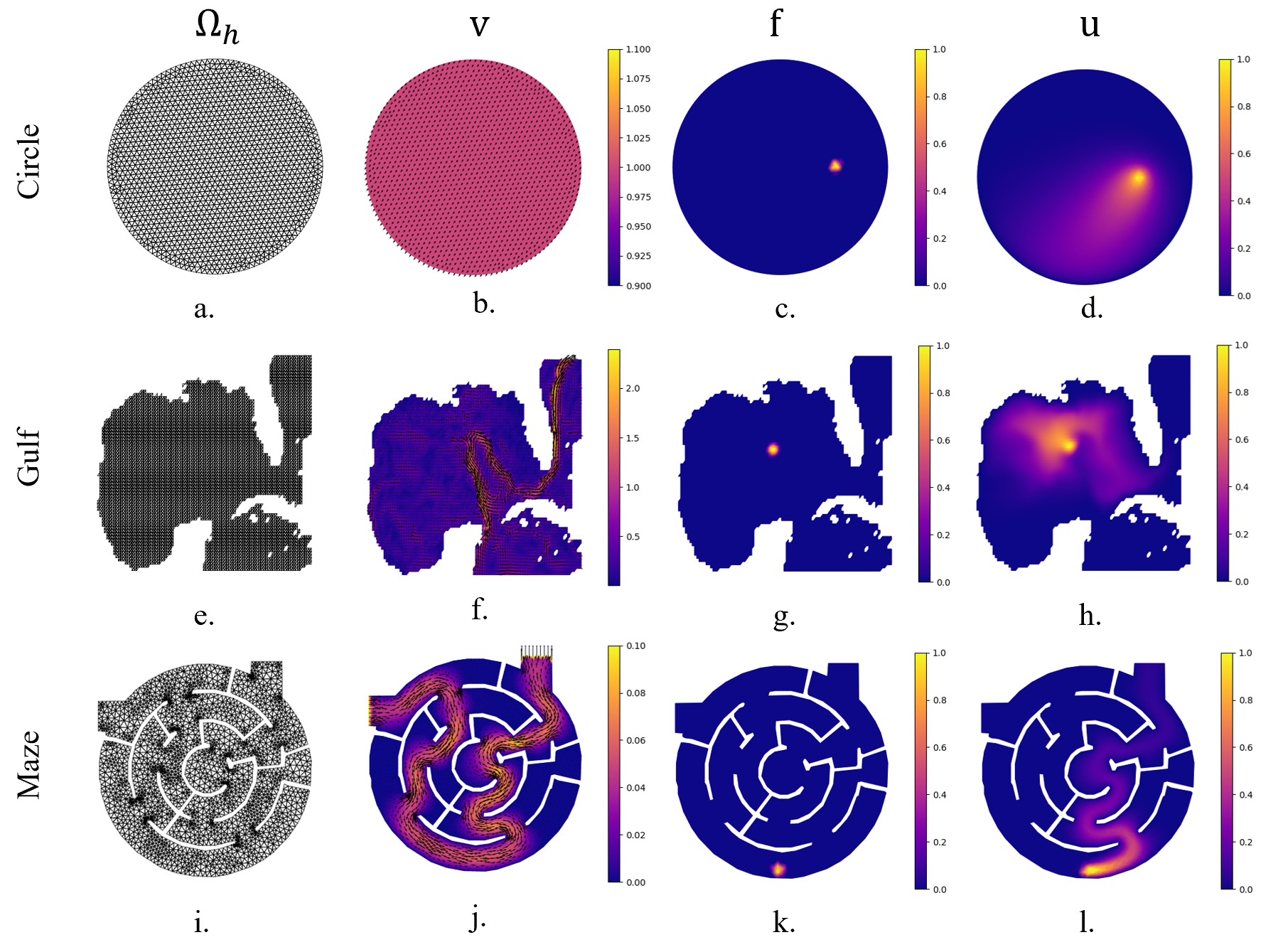}
    \caption{Sample source terms and resulting solutions for advection diffusion in the three cases considered.}
    \label{fig:data_overview}
\end{figure}

\section*{Appendix C: Model and Training Parameters}
\label{app:model_params}

All models use the same general architecture and training setup unless otherwise noted. Each model was implemented in PyTorch and trained in double precision. ReLU activations were used throughout both the transformer encoder and MLP head networks. A constant dropout rate of 0 was used across all layers. All optimization was performed using Adam with default PyTorch settings.

The table below summarizes the model architecture and training hyperparameters used across experiments.

\begin{table}[ht!]
    \centering
    \caption{Model architecture and training parameters used across all experiments.}
    \begin{tabular}{ll}
        \toprule
        \textbf{Parameter} & \textbf{Value} \\
        \midrule
        Number of trainable POUs ($n_{\text{POUs}}$) & 5 \\ 
        Number of transformer heads ($n_{\text{heads}}$) & 4 \\
        Token embedding dimension ($d_{\text{token}}$) & 32 \\
        Latent dimension ($d_{z'}$) & 256 \\
        Transformer MLP width & 128 \\
        MLP head width & 128 \\
        MLP head depth & 6 \\
        Activation function & ReLU \\
        Solver: max iterations & 30 \\
        Solver: tolerance & $10^{-8}$ \\
        Precision & Double \\
        \midrule
        Learning rate & $1 \times 10^{-3}$ \\
        Train epochs & $2 \times 10^{5}$ \\
        Batch size & $64$ \\
        Data cache size & 1600 samples \\
        Data cache reset tolerance & $25 \times$ \\
        Validation set size & 20\% of training cache \\
        \bottomrule
    \end{tabular}
    \label{tab:model-params}
\end{table}

\section{Appendix D: Proof for 
convergence to sensor location}
\label{app:conv_proof}

We provide a proof for Theorem~\ref{thm:conv}

\begin{proof}

\begin{figure}
    \centering
	\includegraphics[width = .3\textwidth]{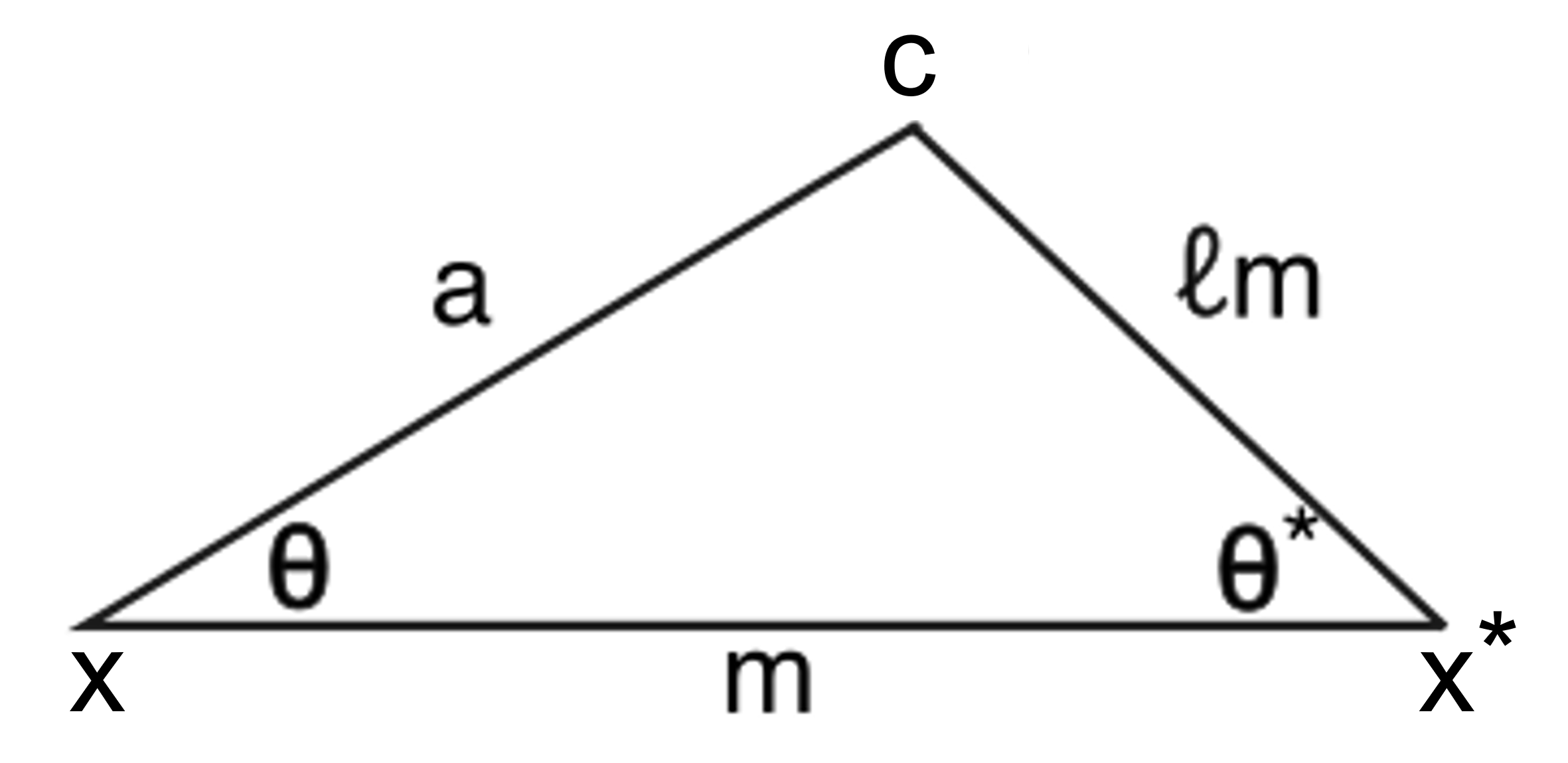}
	\caption{Figure for Theorem \ref{thm:conv}.} \label{figlloyd}
\end{figure}

In (\ref{lipbound}), if we let $X_1 = X$ be an initial configuration of sensors with a minimal distance $m_0< m_{\mathrm{c}}$, and  $X_2$ be a configuration with a sensor located at $x^*$, then the inequality becomes
\begin{equation}
    \|c^*(X)  - x^*\| \le rm_0. \label{lipperbound}
\end{equation}
In Fig. \ref{figlloyd} we use a coordinate system that places $x$ at the origin and $x^*$ on the positive $x$-axis.  We can also take the centroid $c^*$ to be on the upper half plane, so that that the angle $\angle (x,x^*,c^*) = \theta^* \in [0, \pi]$. 

Given values of $\theta^*$ and $\ell$, we will calculate the side length $a = \|c^* - x\|$ and subsequently the rate of change of the minimal distance $m = \|x - x^*\|$. For $\theta = \angle(x^*, x, c^*)$, the law of sines gives
\begin{equation}
	\frac{m}{\sin(\pi- \theta^*- \theta)}  = \frac{\ell m}{\sin(\theta)} = \frac{a}{\sin(\theta^*)} . \label{lawsines}
\end{equation}
From the first equality in (\ref{lawsines}),
\begin{align}
	\sin(\theta) = \ell \sin(\pi- \theta^*- \theta)  = \ell \sin (\theta^*+ \theta) = \ell (\sin(\theta) \cos(\theta^*) + \sin(\theta^*) \cos(\theta)),
\end{align}
and so we relate $\theta$ and $\theta^*$ by 
\begin{equation}
	\cot(\theta) = \frac{1- \ell\cos(\theta^*)}{\ell\sin(\theta^*)}.
\end{equation}
From the second equality in (\ref{lawsines}), 
\begin{equation}
	a = \frac{\ell \sin(\theta^*)}{\sin(\theta)} m. \label{aform}
\end{equation}

Recalling that $m(t) = \min_{k} \| x_k(t) - x^*\|$, we note that the minimum of differentiable functions is absolutely continuous and has a derivative almost everywhere.  Where the derivative exists, we write $x(t) = (x_1(t), x_2(t))$, with $x(t') = x$ and use the Lloyd dynamics (\ref{eq:cont_lloyd_update}) with (\ref{aform}) to find
\begin{align}
    \dot m(t') = \frac{d}{dt} \|x(t) - x^*\|\Bigr|_{t = t'} &= \frac{-m(t')\dot x_1(t')+ x_2(t') \dot x_2(t')}{m(t')} \\ &= -ka\cos(\theta) =  -k(1 - \ell\cos(\theta^*))m \le k(r-1)m,
\end{align}
where the maximum value was found with $\ell = r$ and $\theta^* = 0$.  For $t \ge 0$, 
\begin{equation}
    \dot m(t)  \le \alpha(r-1) m \label{cron1}
\end{equation}
almost everywhere. Since $m(t)$ is absolutely continuous, this inequality can then be integrated to obtain
\begin{equation}
m(t) \le e^{\alpha(r-1)t }m_0 \label{cron2}, 
\end{equation}
and so $m(t) \rightarrow 0$ as  $t \rightarrow \infty$.
\end{proof}

We now provide a general class of functions which satisfy (\ref{lipperbound}) for all configurations $X$, ensuring global convergence of sensors to $x^*$. 
This class of functions places a nonzero mass across the entire domain, to ensure the centroid formula (\ref{eq:centroid}) for cells not containing the event is well-defined and approximately solves the unweighted coverage problem.

The main task in constructing such functions is ensuring that there is a universal lower bound on the fraction of mass from $\rho_X$ contained in $V^*$.  We will use bump functions.  By this, we say that $\varphi$ is a bump function if it is a nonnegative smooth function with support contained in the unit disk centered at the origin and $\int_{\mathbb R^2} \varphi(x)dx = 1$.  It is also required to satisfy a \textit{sector condition}: there exists $\alpha>0$ such that for any sector $K_\theta$ of the unit circle at the origin with angle $\theta$, 
\begin{equation}
    \int_{K_\theta} \varphi dx \ge \alpha \theta.
\end{equation}
The sector condition holds for a large collection of functions.  For instance, $\varphi$ can be radially symmetric, or have a support which contains a disk of any size centered about the origin. 

We consider bump functions containing $x^*$ having support $D :=  D_{r'm}(x^*)$ (a disk of radius $r'm$ centered at $x^*$). Outside of the support, the function is a sufficiently small constant $\beta$ which makes the dynamics and defined, and reduces to unweighted Lloyd dynamics in cells outside of the bump function's support.  For a bump function $\varphi$ and $r' \in (0,1)$, the importance function takes the form of a translated and scaled copy of $\varphi$ with background noise:
\begin{equation} \label{support}
    \rho_P(x) =
    \begin{cases}
        \frac{1}{(r'm)^2}\varphi\left(\frac{q}{r'm}\right ) & q \in D, \\ 
        \beta &q \in \Omega \backslash D. 
    \end{cases}
\end{equation}
Here, $\beta$ is a small constant which depends only on $m$, $r'$, and the diameter of the domain $M$.

In what follows, we fill find a sufficiently small $\beta$ such that convergence holds under sensors following (\ref{eq:cont_lloyd_update}).  
This involves showing that any configuration of sensors with minimal distance $m$ is guaranteed to contain a uniform lower bound of the fraction of $D$ in the Voronoi cell $V^*$ which contains $x^*$. 
Let $x$ be the sensor for $V^*$. By convexity of Voronoi cells, the interior of the segment $\overline{xx^*}$ is contained in the interior of $V^*$ (this is true even if $x^*$ is in the boundary of $V^*$). In Fig. \ref{figfrac}, we use coordinates so that this segment lies on the $x_1$-axis, with $x$ placed at the origin and to the left of $x^*$.  On the segment perpendicular to $\overline{xx^*}$ intersecting $\partial D$ intersects $\overline{xx^*}$, denote  $q_1$ and $q_2$ as the closest points above and below the x-axis, respectively, which intersect the boundary of $V^*$. These points may either be on the boundary of $\Omega$ or bordering another Voronoi cell.  By convexity, the quadrilateral $\square (q_1,x,q_2,x^*) \subset V^*$. Define the angles $\theta = \angle (q_1, x^*, q_2)$ and $\theta^* = \angle (q_1, x, q_2)$.

\begin{figure}
    \centering
	\includegraphics[width = .4\textwidth]{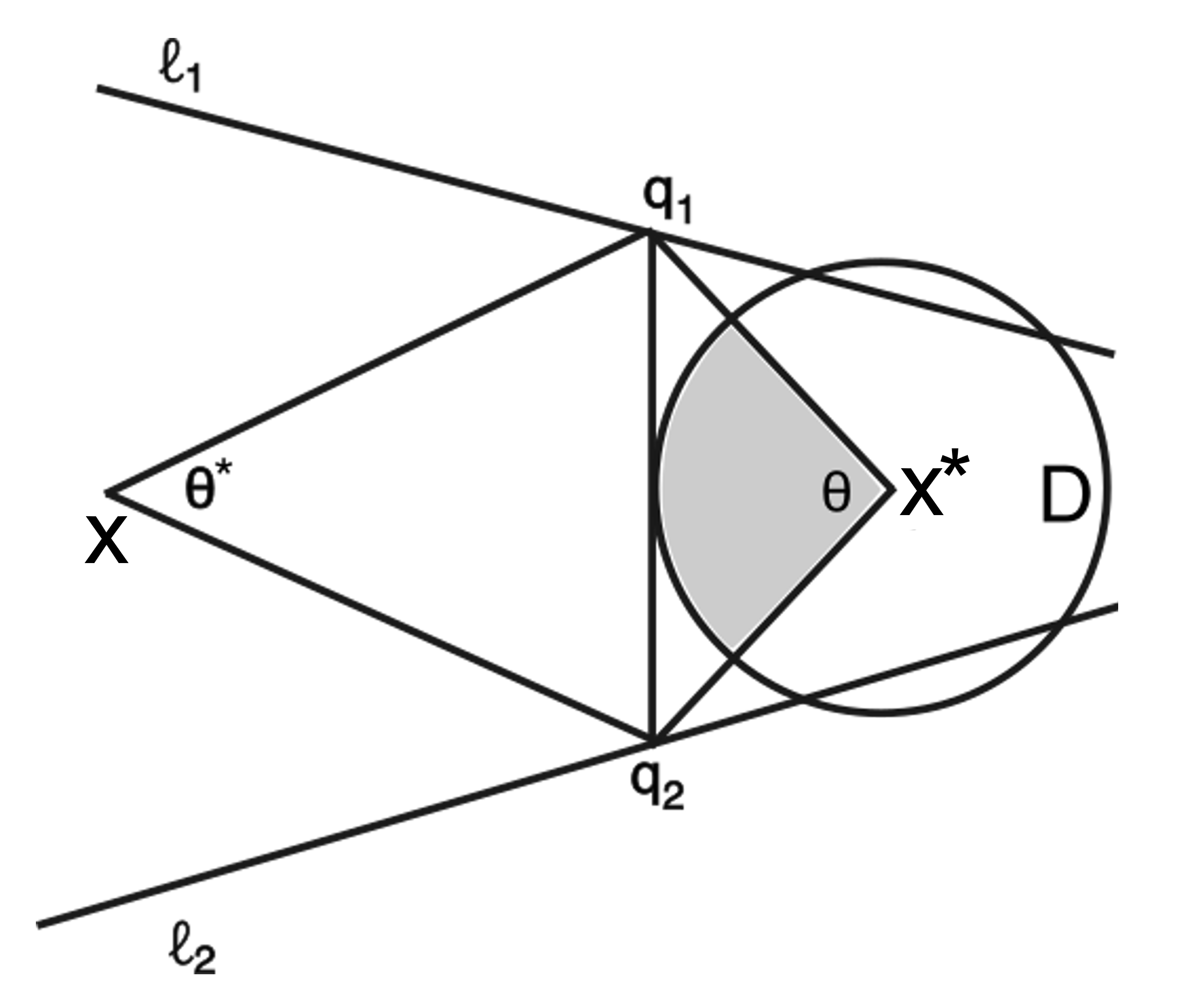}
	\caption{Figure for Lemmas \ref{invthm} and \ref{cormin}.}
    \label{figfrac}
\end{figure}

\begin{lemma} \label{invthm} The following bounds hold:

\vspace{5pt}

(i) $|D \cap V^*|\ge \frac{\theta}{2}(mr')^2, $

(ii) $|V^*| \le \frac{R^2}{2}(\theta^* +\theta)$. 
\end{lemma}

\begin{proof}

To show (i), we note that, by convexity, the triangle $\triangle (q_1, q_2, x^*)\subset V^*$, and so is the sector of $D$ contained in the triangle (shaded in the figure), which has area $\frac{\theta}{2}(mr')^2$.  

To show (ii), we recall that a Voronoi cell is an intersection  of half-planes, determined by lines which are perpendicular bisectors of segments connecting site points. This implies that $V^*$ is contained in the region containing $x$ and bounded by $\Omega$ and any two lines $\ell_i$ which contain  $q_i$ and a side of $V^*$ or the boundary of $\Omega$.  These lines cannot intersect the interior of $\square (q_1,x,q_2,x^*)$.  To the right of the segment $\overline{q_1q_2}$, $V^*$ has an area at most $\theta^* R^2/2$. This corresponds to the case when $\ell_1$ intersects $\overline{xq_1}$ and $\ell_2$ intersect $\overline{xq_2}$.  Similarly, to the left of $\overline{q_1q_2}$, the area is bounded by $\theta^* R^2/2$, which gives
\begin{equation}
    |V^*| \le \frac{R^2}{2}(\theta^* +\theta).
\end{equation}

\end{proof}

The bounds for $|V^*|$ and $|V^* \cap D|$ can now be used to show that for a sufficiently small value of $\beta(r',m,R)$, the importance function places enough weight near $x^*$ so that the distance between $x^*$ and $x^*$ is less than $m$, and thus satisfies (\ref{lipperbound}).

\begin{lemma} \label{cormin}
There exists a constant $\beta = \beta(\alpha,R, m, r')$ in (\ref{support}) that can be chosen to be small enough so that for any configuration $P$ of sensors and $r \in (r',1)$,
\begin{equation}
	\|c^*(\rho_X)- x^*\| \le rm.\label{mincond} 
\end{equation} 
\end{lemma}

\begin{proof}

If we denote
\begin{equation}
    \varphi_{m,r'}(q) = \frac{1}{(mr')^2}\varphi(q/mr'),
\end{equation} 
then the mass $M^*$ of $V^*$ is given by 
\begin{align}
    M^*(\rho) = \int_{V^* \cap D} \varphi_{m,r'}(q)\;dq + \int_{V^*\backslash D} \beta \; dq:= I_1 + I_2.
\end{align}

We now show that as $\beta\rightarrow 0$, $I_1/I_2\rightarrow 0$.  Indeed, from the sector condition and Lemma (\ref{invthm}),
\begin{align}
    \frac{I_2}{I_1}  \le \frac{\beta R^2}{\alpha(mr')^2}\left(1+ \frac{\theta^*}{\theta} \right).
\end{align}
For  $\theta \in (0, \pi_2)$,  $\theta^*/\theta$ can be uniformly bounded from above by a constant which only depends on $r'$.  This is clearly true for $\theta \in (\pi/3, \pi/2)$.  For $\theta \in (0, \pi/3)$, we consider the angles $\theta_1 = \angle(q_1, x^*, x)$ and $\theta_1^* = \angle(q_1, x^*, \theta_1^*)$ lying above the $x_1$-axis. Using the simple bounds $\arctan(\theta) \le \theta$ and $\tan(\theta) \le 2 \theta$, it is straightforward to show
\begin{equation}
    \theta_1^* = \arctan\left(\frac{r'}{1-r'}\tan(\theta_1)  \right)\le \frac{2r'}{1-r'} \theta_1. 
\end{equation}
A similar calculation follows for calculating angles below the $x_1$-axis. Then $\theta^*/\theta \le 2r'(1-r')$, and so $I_2/I_1$ can be made uniformly small by choosing a value of $\beta$ dependent only on $R,\alpha, r'$ and $m$.  

The centroid $c^*$ is invariant under multiplying $\rho$ by a constant.  If we let $\tilde \rho = \rho/I_1$, then $M^*(\tilde \rho) = 1+ I_2/I_1$, and as $\beta\rightarrow 0$,
\begin{align}
    c^*(\rho) = c^*(\tilde \rho) &= \frac{1}{1+ I_2/I_1}\left(\int_{V^* \cap D} \varphi_{m,r'}(q)/I_1\;dq + \frac{\beta}{I_1}\int_{V^*\backslash D} q\;dq\right)
    \\
    &\rightarrow \int_{V^* \cap D} \varphi_{m,r'}(q)/I_1\;dq = c^*(\varphi_{m,r'}) \in D.
\end{align}
It follows that we can choose $\beta(\alpha,R, m, r') $ small enough such that (\ref{mincond}) holds.
\end{proof}




\end{document}